\documentclass[11pt,a4paper]{article}

\usepackage[T1]{fontenc}
\usepackage[utf8]{inputenc}
\usepackage{lmodern}
\usepackage[margin=1in]{geometry}
\usepackage{graphicx}
\usepackage[round,authoryear]{natbib}
\usepackage{comment}
\usepackage{amsmath,amssymb,amsfonts,amsthm,mathtools}
\usepackage{bm}
\usepackage{booktabs,tabularx}
\usepackage{enumitem}
\usepackage{microtype}
\usepackage{xcolor}
\usepackage{algorithm,algpseudocode}
\usepackage{hyperref}
\usepackage[capitalise,noabbrev]{cleveref}

\hypersetup{
  colorlinks=true,
  linkcolor=blue!55!black,
  citecolor=blue!55!black,
  urlcolor=blue!55!black
}

\theoremstyle{plain}
\newtheorem{theorem}{Theorem}
\newtheorem{lemma}[theorem]{Lemma}

\newtheorem{corollary}[theorem]{Corollary}

\theoremstyle{definition}

\newtheorem{example}{Example}

\theoremstyle{remark}

\crefname{theorem}{Theorem}{Theorems}
\Crefname{theorem}{Theorem}{Theorems}
\crefname{lemma}{Lemma}{Lemmas}
\Crefname{lemma}{Lemma}{Lemmas}
\crefname{proposition}{Proposition}{Propositions}
\Crefname{proposition}{Proposition}{Propositions}
\crefname{corollary}{Corollary}{Corollaries}
\Crefname{corollary}{Corollary}{Corollaries}
\crefname{definition}{Definition}{Definitions}
\Crefname{definition}{Definition}{Definitions}
\crefname{assumption}{Assumption}{Assumptions}
\Crefname{assumption}{Assumption}{Assumptions}
\crefname{problem}{Problem}{Problems}
\Crefname{problem}{Problem}{Problems}
\crefname{example}{Example}{Examples}
\Crefname{example}{Example}{Examples}
\crefname{remark}{Remark}{Remarks}
\Crefname{remark}{Remark}{Remarks}

\title{Runtime Monitoring of Supermartingale Certificates}
\author{%
  Mirco Giacobbe\textsuperscript{1}\quad
  Thomas A.~Henzinger\textsuperscript{2}\quad
  Konstantin Kueffner\textsuperscript{5}\\
  Emily Yu\textsuperscript{3}\quad
  \DJ{}or\dj{}e \v{Z}ikeli\'{c}\textsuperscript{4}\\[0.8em]
  {\small\textsuperscript{1}University of Birmingham}\\
  {\small\textsuperscript{2}Institute of Science and Technology Austria}\\
  {\small\textsuperscript{3}Leiden University}\\
  {\small\textsuperscript{4}Nanyang Technological University}\\
  {\small\textsuperscript{5}University of Oxford}%
}
\date{}

\newcommand{\gM}{\mathfrak{M}}

\newcommand{\aA}{\mathcal{A}}

\newcommand{\aF}{\mathcal{F}}
\newcommand{\aG}{\mathcal{G}}

\newcommand{\aS}{\mathcal{S}}

\newcommand{\aX}{\mathcal{X}}
\newcommand{\aY}{\mathcal{Y}}

\newcommand{\sE}{\mathbb{E}}

\newcommand{\sN}{\mathbb{N}}

\newcommand{\sP}{\mathbb{P}}

\newcommand{\sR}{\mathbb{R}}

\newcommand{\sT}{\mathbb{T}}

\newcommand{\rV}{V}

\newcommand{\rX}{X}
\newcommand{\rY}{Y}

\newcommand{\pX}{\bm{X}}
\newcommand{\pY}{\bm{Y}}

\newcommand{\expe}{\sE}
\newcommand{\prob}{\sP}

\newcommand{\NN}{\sN}

\newcommand{\RN}{\sR}

\newcommand{\moni}{\gM}

\newcommand{\sample}{\pi}
\newcommand{\verdict}{\xi}

\begin{document}
\raggedbottom
\maketitle

\begin{abstract}
Supermartingale certificates formally guarantee temporal properties of discrete and continuous-time stochastic dynamical systems. However, these guarantees rely on a mathematical model of the world, leaving a sim-to-real gap that can invalidate them for the deployed system. We formulate the runtime monitoring problem for supermartingale certificates, in which a monitor observes a single system trajectory and raises an alarm when there is sufficient statistical evidence that the supermartingale certificate has been violated during execution. We present an algorithm to automatically construct a monitor using betting-based sequential tests that bound the probability of ever raising a false alarm at a user-specified level. Within this betting framework, we formulate strategies that allow the monitor to exploit spatial regularity to improve detection power while preserving this statistical guarantee. We demonstrate the efficacy of our approach on supermartingale certificate benchmarks from the literature.
\end{abstract}

\section{Introduction}\label{sec:intro}

Supermartingale certificates guarantee that probabilistic, potentially infinite-state systems satisfy temporal properties of their intended behaviour, making them a powerful tool for assuring sequential decision-making under uncertainty. Applications include the automated control of unmanned vehicles, manufacturing processes, and financial systems, where rigorous guarantees of safety and security are paramount. Supermartingale certificates are real-valued functions of system states whose expected values decrease along system executions in a manner determined by the specification of interest. They are the probabilistic analogues of Lyapunov and barrier functions in control systems~\citep{kushner1965stability}, and of inductive invariants and progress measures in software systems~\citep{Floyd1967,ChakarovS13}. 
A supermartingale certificate reduces reasoning about a global temporal property of system behaviour---such as reachability, safety, reach-avoidance, stability, reactivity---to verifying a local expected-decrease condition at each individual step of the system. 

The algorithmic construction of supermartingale certificates has traditionally relied on rigorous symbolic reasoning methods~\citep{PrajnaJ04,DBLP:conf/cav/ChatterjeeFG16}, with more recent work turning to neural representations and certified learning-based approaches~\citep{AbateGR20,LechnerZCH22,ZikelicLHC23,NeustroevGL25}. 
Yet, the formal guarantees provided by the existing methods are contingent on a given {\em model of the world}, fixed at design time. If the behaviour of the environment in which the controller is deployed deviates from this assumed model, the certificate may no longer satisfy the supermartingale conditions, invalidating its guarantees. Such deviations may arise, for example, from environmental changes, distribution shifts, or unmodelled physical conditions, and are commonly associated with what is known in control theory and robotics as the {\em sim-to-real gap}~\citep{JacobiHH95,TobinFRSZA17,ZhaoQW20}. 

This raises a fundamental question: when the true environment is not fully known and can only be interacted with at runtime as a black box, how can we determine whether a certificate remains valid after deployment? This is the central question we address in this work:
we propose for the first time a framework for the {\em runtime monitoring of supermartingale certificates} upon deployment. Given a supermartingale certificate computed at design time and a a black-box stochastic deployment environment, our method observes a single stream of system states, continually tests whether the supermartingale certificate satisfies its defining expected-decrease condition, and raises an alarm once it has gathered sufficient statistical evidence that the supermartingale certificate is invalid.

Runtime monitoring of a supermartingale certificate requires detecting violations from a single stream of observations. Because the environment is stochastic, an increase (resp.~decrease) in the certificate value at any individual step does not necessarily indicate a violation (resp.~satisfaction) of the expected-decrease condition. To address this challenge, our monitor accumulates statistical evidence over successive state observations, seeking to determine when the observed changes in certificate value become sufficiently inconsistent with the expected-decrease condition to signal a violation. This evidence is accumulated using a {\em betting-based sequential testing method}~\citep{ramdas2023game,waudby2024betting}. The monitor plays a betting game: it starts with an initial wealth and, for each pair of successive state observations, converts the change in certificate value into a normalised score and chooses a stake based on the observations seen so far. A positive score increases the monitor's wealth, while a negative score decreases it, with the magnitude of either change determined by the chosen stake. Crucially, while the expected-decrease condition holds, no admissible betting strategy can increase the monitor's wealth in expectation. Sustained wealth growth therefore constitutes statistical evidence against this condition. Once the wealth exceeds a threshold calibrated to control the false-alarm probability, the monitor signals a violation. 

We implemented our method and evaluated betting strategies that either weight past observations uniformly or favour recent observations and those from nearby states across a wide range of benchmarks. Our results show that the framework effectively detects certificate violations under changes in deployment dynamics, and that exploiting temporal and spatial regularity in the betting strategy improves the detection of local violations.

Our contribution is threefold. First, we present the first framework for runtime verification of supermartingale certificates. Our framework is agnostic to the temporal property being certified, applies to both discrete- and continuous-time stochastic systems, and encompasses the existing supermartingale certificates (see Section~\ref{sec:prelims}). Second, we introduce an algorithm for automatically constructing monitors based on betting games, with a user-specified bound on the false-alarm probability that holds uniformly over an unbounded monitoring horizon. This bound allows users to control the monitor's conservativeness and avoid unnecessary intervention (see Section~\ref{sec:monitor-construction}).
Finally, we empirically demonstrate the effectiveness of our method in detecting certificate violations under changed
deployment dynamics (see Section~\ref{sec:exp}).

\section{Runtime monitoring of supermartingale certificates}\label{sec:prelims}

\paragraph{Markov Process} 
We consider probabilistic systems modelled as Markov processes in either discrete or continuous time, over which supermartingale certificates are defined. We generically write $\sT$ for the time domain, with $\sT=\NN_0$ (the non-negative integers) in discrete time and $\sT=\RN_{\geq 0}$ (the non-negative reals) in continuous time, and take $t=0$ as the initial time.

Let $\aX = \mathbb{R}^d$ be a Euclidean space, and $(\aX, \mathcal{F}, \mathbb{P})$ be a probability space over $\aX$. A property holds \emph{almost surely} if it holds with probability one. Let $(\aG_t)_{t\in\sT}$ be a filtration in this probability space. A {\em stochastic process} $\pX=(\rX_t)_{t\in\sT}$ with respect to $(\aG_t)_{t\in\sT}$ is
an indexed family of random variables, with each $\rX_t$ being $\aG_t$-measurable. In the continuous-time setting, we also assume that $\pX$ is right-continuous and that the filtration $(\aG_t)_{t\in\sT}$ is complete and right-continuous. A {\em (time-homogeneous) Markov process} $\pX=(\rX_t)_{t\in\sT}$ is a stochastic process such that, for every measurable function $f: \aX \rightarrow \mathbb{R}$ and for every pair of time values $t,s \in \sT$, we have $\mathbb{E}[f(\rX_{t+s}) \mid \aG_t] = \mathbb{E}[f(\rX_{t+s}) \mid \rX_t]$.

\paragraph{Supermartingale Certificate} A {\em supermartingale certificate} is a measurable function $f: \aX \rightarrow \mathbb{R}$ which satisfies a finite number of conditions of one of the following two types:
\begin{enumerate}[leftmargin=*]
    \item {\em domain-bound condition.} Given a domain $\aS \subseteq \aX$ and a bound interval $[l,u]$ where $l, u \in \mathbb{R} \cup \{\pm\infty\}$, the $(\aS ,[l,u])$-domain-bound condition specifies that, for all times $t \in \sT$, we have $f(\rX_t) \in [l,u]$. Intuitively, this condition requires the value of $f$ to lie within the interval $[l,u]$ at all times, whenever the stochastic process value is in the domain set $\aS $.
    
    \item {\em expected-decrease condition.} 
 Given a domain $\aS\subseteq\aX$ and a measurable drift bound $l:\aX\rightarrow\mathbb{R}$, zero outside $\aS$, the $(\aS,l)$-expected drift condition requires, for $t\in\sT$ and $\rX_t\in\aS$, a bound on
\begin{itemize}[leftmargin=*]
    \item ({\em discrete-time}) the one-step expected decrease: $\mathbb{E}[f(\rX_{t+1})\mid\rX_t]-f(\rX_t)\leq l(\rX_t)$;
    \item ({\em continuous-time}) the instantaneous drift: $ \mathcal{A}f(\rX_t)
        \coloneqq 
        \lim_{h\downarrow0}
        \frac{\mathbb{E}[f(\rX_{t+h})\mid\rX_t]-f(\rX_t)}{h}
        \leq l(\rX_t)$. 
\end{itemize}
Intuitively, this bounds the expected change per step, or the instantaneous drift rate, from above by $l$ whenever the process is in $\aS$. The continuous-time setting requires additional regularity assumptions specified in Appendix~\ref{app:sec:continuous}. 
\end{enumerate}
These defining conditions subsume all classical supermartingale certificates for proving temporal properties in stochastic systems, such as ranking supermartingales~\citep{ChakarovS13} for almost-sure reachability, stochastic barrier functions~\citep{PrajnaJP07} for probabilistic safety, reach-avoid supermartingales~\citep{ZikelicLHC23} for probabilistic reach-avoidance, or Streett supermartingales~\citep{AbateGR24,AbateGR25} or limit-deterministic B\"uchi supermartingales~\citep{HenzingerMSZ25} for general omega-regular specifications. Thus, we follow this unified presentation to design a runtime monitoring framework that applies to all these supermartingale certificates.

\begin{example}[A robot climbing a hill]
\label{ex:hill-rsm}
A robot starts at distance $x_0=10$ from its target. Each step moves it
forwards by $0.01$ or makes it slip backwards by the same distance. It stops upon reaching a distance in $[0,1]$ from the target, where it remains. A sloping section between positions $3$ and $7$ increases the probability of slipping backwards. As the robot moves from $7$ towards $3$, the slope becomes steeper until position $5$, then gradually
flattens. We describe this steepness profile by $ b(x)\coloneqq\left(1-|x-5|/2\right)_+.$ Thus, position $5$ is the steepest part of the slope. 
Before reaching the target, the dynamics for $a\in[0,3/4]$ and $p_a(x)\coloneqq1/4+a \cdot b(x)$ are
\begin{align*}
    \rX_{n+1}
    &=\rX_n-0.01+0.02 \cdot Z_n,
    &
    Z_n\mid\aG_n
    &\sim\operatorname{Bernoulli}\bigl(p_a(\rX_n)\bigr).
\end{align*}
Each draw uses fresh randomness with slip probability determined by the current state. The range of $a$ ensures that $p_a(x)\in[0,1]$. If we choose $f(x)=x$ and $l(x)=-0.005 \cdot \mathbf{1}\{x>1\}$, the expected decrease before stopping is $d_a(x)\coloneqq
    \expe[f(\rX_{n+1})-f(\rX_n)\mid\rX_n=x]
    =-0.005+0.02 \cdot a \cdot b(x)$. If $a=0$, $f$ is a ranking supermartingale proving almost-sure target reachability. However, for every $a>0$, the expected-decrease condition is violated on $(3,7)$. 
\end{example}

\paragraph{Trajectory} Given a Markov process $\pX = (\rX_t)_{t \in \sT}$ and a measurable function $f: \aX \rightarrow \mathbb{R}$, the goal of the supermartingale certificate monitoring problem is to construct a monitor which observes states of the Markov process at different time steps and, upon each newly observed state, issues a verdict on whether the function $f$ violates one of the supermartingale conditions or not. To that end, we define $\aY\coloneqq\sT\times\aX$ to be the set of time--state pairs, and use $\aY^\ast$ to denote the set of all finite sequences of time-state pairs. We call elements of $\aY^\ast$ {\em trajectories}.

\paragraph{Monitor} A {\em monitor} is a tuple $\moni\coloneqq(\sample,\verdict)$,
where $\sample\colon\aY^*\to\sT$ is a {\em sampling rule} which for each sampled trajectory specifies the next observation time, and $\verdict\colon\aY^*\to\{0,1\}$ is a {\em verdict rule} which for each sampled trajectory decides whether to raise an alarm or not. The verdict value $1$ denotes an {\em alarm}, indicating statistical evidence of a certificate violation, while $0$ means that no alarm has been raised. The semantics of a monitor are defined inductively as follows. At the beginning, the monitor starts with a single-state trajectory $\pY_{0:0} = (T_0,x_{T_0})$, with $T_0 = 0$ and $x_{T_0}$ sampled from $\rX_{T_0}$. Then, for each $n$, if $\pY_{0:n}=(\rY_0, \dots, \rY_n)$ is the trajectory of time-state samples that have been sampled after $n$ steps, the next time-state sample and the monitor verdict are computed via
\begin{align*}
  T_{n+1}
  \coloneqq
  \sample(\pY_{0:n}),
  \quad
  x_{T_{n+1}} \sim X_{T_{n+1}},
  \quad
  \pY_{0:n+1} \coloneqq (\pY_{0:n}, (T_{n+1}, x_{T_{n+1}})),
  \quad 
  \rV_{n+1}
  \coloneqq
  \verdict(\pY_{0:n+1}).
\end{align*}
Hence, the monitor determines the next sampling time $T_{n+1}$, samples the next state $x_{T_{n+1}}$ from $\rX_{T_{n+1}}$, appends the new time-state pair to the trajectory to obtain trajectory $\pY_{0:n+1}$, and finally computes the verdict $\rV_{n+1}$ by applying the verdict rule to the trajectory $\pY_{0:n+1}$.

\paragraph{Problem Statement} Our goal is to design a monitor that raises an alarm as accurately as possible, while providing {\em probabilistic guarantees} that its false alarm rate is below some tolerable threshold $\delta$ {\em over all sampling times}, and not just for each individual sampled state. The accuracy of our monitor is an empirical claim that will be evaluated in our experiments. However, the false alarm rate is a property that we want to {\em guarantee}. We define it as follows. Given a candidate certificate $f$ and an error tolerance $\delta\in(0,1)$, we want to construct a monitor $\moni\coloneqq(\sample,\verdict)$ such that, whenever $f$ satisfies the certificate conditions for the deployed process, we have $\prob\left(\exists n\in\NN_0:\rV_n=1 \right)\leq\delta$.

\paragraph{Challenge} 
To determine whether $f$ is a supermartingale certificate, it suffices to design one monitor for each domain-bound condition and each expected-decrease condition of the supermartingale. But designing a monitor for domain-bound conditions is trivial -- if $S \subseteq \aX$ is the domain and $[l,u]$ is the bound interval, the verdict rule is simply $\rV_{n} = \verdict(\pY_{0:n}) = 1$ iff the last time-state pair $(T_n,x_n)$ satisfies $x_n \in S$ but $f(x_n) \not\in [l,u]$. Such monitor design ensures that the false alarm rate is $0$. 

Hence, the challenge in supermartingale certificate monitoring is to compute a monitor for the expected-decrease condition, and in the rest of this work we focus on that.
Specifically, our construction uses fixed observation times $t_n=nh$ and bounds the probability of raising an alarm before any completed observation interval has positive conditional expected corrected drift.
Formally, let $L_n\coloneqq\sum_{r=t_n}^{t_{n+1}-1}l(\rX_r)$ in discrete time and $L_n\coloneqq\int_{t_n}^{t_{n+1}}l(\rX_r)\,\mathrm dr$ in continuous time denote the {\em accumulated allowed drift}.
Writing $\tau_\delta$ for the first alarm index and
$E_n\coloneqq\expe[f(\rX_{t_{n+1}})-f(\rX_{t_n})-L_n\mid\aF_n]$
for the conditional expected corrected drift, we prove
\begin{align*}
    \prob\left(
        \tau_\delta<\infty
        \ \text{and}\
        E_n\leq0\text{ for every }n<\tau_\delta
    \right)\leq\delta.
\end{align*}

\section{Runtime monitoring as betting-based sequential testing}
\label{sec:monitor-construction}
We use a betting-based sequential testing framework~\citep{ramdas2023game} as the main underlying ingredient to construct a monitor for supermartingale certificates. 
In this section, we first show how a sequence of observed states can be transformed into scores that provide evidence of a certificate violation. We then define a general monitoring scheme, namely an \emph{abstract monitor}, that guarantees false-alarm control for any admissible betting strategy.  Lastly, we instantiate the abstract monitor with concrete strategies designed to improve detection power while preserving monitor validity.

A supermartingale $(M_n)_{n\in\NN_0}$ is \emph{nonnegative}
if $M_n \ge 0$ almost surely for every $n\ge 0$.
In the betting-based framework, we use \emph{Ville's inequality} to turn such a process into a monitor
with a bound on the probability of ever raising a false alarm. Ville's inequality states that, for any nonnegative supermartingale $(M_n)_{n\in\NN_0}$ with $M_0=1$ and for every $\delta\in(0,1)$, the probability of the supermartingale value ever reaching or exceeding $1/\delta$ is at most $\delta$, i.e.,
\begin{align}
    \prob\left(
        \exists n\in\NN_0:M_n\geq 1/\delta
    \right)\leq\delta.
    \label{eq:ville-bound}
\end{align}
We set the first alarm index $\tau_\delta$ and the verdict $\rV_n$ to
\begin{align*}
    \tau_\delta
    \coloneqq\inf\{n\in\NN_0:M_n\geq1/\delta\},
     \quad\text{and} \quad 
    \rV_n \coloneqq\mathbf{1}\{\tau_\delta\leq n\} \quad \text{with $\inf\varnothing\coloneqq\infty$.}
\end{align*}
This defines a statistical test, such that whenever the null hypothesis ensures that $(M_n)_{n\in\NN_0}$ is a nonnegative supermartingale, the probability of ever raising an alarm is at most $\delta$.

\subsection{Scoring schemes}

To apply the betting framework, we assign a numerical score to
each pair of consecutive observations. A positive score indicates
that the observed certificate change exceeds the allowed drift
upper bound. However, a positive score can also occur by chance.
We construct the scores so that, whenever the 
expected-decrease condition holds, the expected next score,
given the observations so far, is at most zero.
This ensures that betting on these scores cannot increase wealth
in conditional expectation while the condition holds.
We also ensure that each score is at least $-1$ almost surely,
so that admissible bets cannot make the wealth negative.

\paragraph{Assumptions}
The monitor observes the Markov process $\pX$ at times
$t_n=nh$, with fixed intervals $h\in\RN_{>0}$ in continuous time and $h=1$
in discrete time. 
We write $\aF_n\coloneqq\sigma(\rX_{t_0},\dots,\rX_{t_n})$ for the information available to the monitor at $t_n$.
Furthermore, in continuous time, the deployed process is absorbed
at its first exit from $\aS$, retaining its terminal state. For example, for a reach--avoid certificate, the system terminates upon reaching the target or an unsafe state. 

We assume that the monitor can compute two bounds from its
observations. First, after observing a state at $t_{n+1}$, it can
compute a finite upper bound $\overline{L}_n\ge L_n$ almost
surely on the accumulated allowed drift. 
Subtracting this bound
from the observed certificate change gives the
\emph{corrected certificate change}
$D_n\coloneqq f(\rX_{t_{n+1}})-f(\rX_{t_n})-\overline{L}_n$.
Second, at $t_n$ and before observing state $\rX_{t_{n+1}}$, the monitor must already be able to compute a finite \emph{loss bound} $B_n\ge b_{\min}$,
where $b_{\min}>0$ is fixed, guaranteeing $D_n\ge -B_n$
almost surely. This ensures that the normalised score
$D_n/B_n$ is at least $-1$.
Formally, $\overline{L}_n$ is $\aF_{n+1}$-measurable and $B_n$
is $\aF_n$-measurable, and these requirements hold for every
$n\in\NN_0$.
Moreover, we assume that the random variables $f(\rX_{t_n})$, $L_n$, and $\overline{L}_n$ are integrable.

Concrete conservative bounds can be obtained if we know a lower bound on the certificate value $f(x)\geq f_{\min}$ and an upper bound on the allowed drift $l(x)\leq l_{\max}$ for all reachable states. Then we get $ \overline L_n=hl_{\max}$ and $ B_n=\max\{b_{\min},f(\rX_{t_n})+hl_{\max}-f_{\min}\}$.

\paragraph{Scoring}
The monitor compares the observed certificate change with the upper bound on the allowed drift and normalises it by the loss bound. A positive score therefore indicates that the observed certificate change exceeds the drift upper bound. The score is bounded below by $-1$ and has nonpositive conditional mean whenever the certificate condition holds.
Formally, upon observing a state at $t_{n+1}$, the monitor computes a raw score $D_n$ and the normalised score $W_n$:
\begin{align*}
    D_n&\coloneqq f(\rX_{t_{n+1}})-f(\rX_{t_n})-\overline{L}_n,
    &
    W_n&\coloneqq D_n/B_n.
\end{align*}
Let $\mu_n\coloneqq\expe[W_n\mid\aF_n]$ denote the expected
next score given the observations so far. The following lemma
shows that $\mu_n\le 0$ whenever the supermartingale
expected-decrease condition holds. Together with the lower bound $W_n\ge -1$ ensured by normalisation,
this makes the scores \emph{valid} for the betting construction, where
admissible bets preserve nonnegative wealth and cannot increase
it in conditional expectation while the certificate condition
holds. The proof is deferred to Appendix~\ref{app:proofs}.

\begin{lemma}[Score validity]
\label{lem:score-validity}
Recall that
$E_n\coloneqq\expe[
f(\rX_{t_{n+1}})-f(\rX_{t_n})-L_n\mid\aF_n]$
is the conditional expected corrected drift over
$[t_n,t_{n+1}]$. For every $n\in\NN_0$,
\begin{align}
    \mu_n
    =\frac{E_n-\expe[\overline L_n-L_n\mid\aF_n]}{B_n}
    \leq\frac{E_n}{B_n}
    \qquad\text{almost surely}.
    \label{eq:observable-violation}
\end{align}
Thus, if the expected-decrease condition holds and $E_n\leq0$, then
$\mu_n\leq0$ almost surely on $\{\rX_{t_n}\in\aS\}$.
\end{lemma}

\begin{example}
\label{ex:hill-scores}
Consider the hill robot in Example~\ref{ex:hill-rsm}, with observation interval $h=1$. Since $f$ is nonnegative on reachable states and $l\leq0$, valid bounds are $\overline L_n=0$ and $B_n=\max\{b_{\min},f(\rX_{t_n})\}$. 
These bounds only assume that the certificate cannot decrease by more than its current value. However, we can use the known system dynamics to derive more informative bounds. 
The allowed drift is exactly $l(\rX_{t_n})$, and a forward step decreases the certificate by only $0.01$. Before stopping, the corrected increment is therefore at least $-0.01+0.005=-0.005$. We can thus choose $\overline L_n=l(\rX_{t_n})$ and $B_n=b_{\min}=0.02$, giving
\begin{align*}
  W_n=\frac{f(\rX_{t_{n+1}})-f(\rX_{t_n})+0.005}{0.02}
    =Z_n-\tfrac14,
    \qquad \mu_n=a  \cdot b(\rX_{t_n}).
\end{align*}
Hence, a forward step results in $W_n=-1/4$, while a slip gives $W_n=3/4$, so $W_n\geq-1$ for every amplitude, which validates the loss bound. After stopping, $W_n=0$.
\end{example}

\subsection{Runtime monitoring under general betting strategies}
\label{sec:abstract-monitor}
We describe how the monitor accumulates evidence from past
scores and decides to raise an alarm. 
Our monitor combines multiple wealth accounts and 
alarms when its total wealth reaches $1/\delta$. 
The construction here is \emph{abstract} because we do not yet specify the betting
strategies in this subsection.

A single positive score may occur by chance, so the monitor
accumulates evidence by betting on successive scores,
following~\cite{ramdas2023game}.
At each observation index $s\in\NN_0$, we may launch a
set of betting strategies $\aA_s$, called \emph{experts}.
For each expert $a\in\aA_s$, let $G^a_{s,n}$ denote its
normalised wealth at index $n\ge s$, starting from $G^a_{s,s}=1$.
The capital allocated to these accounts is specified below. For every $n\geq s$, if $\rX_{t_n}\in \aS$, expert $a$ chooses an $\aF_n$-measurable stake $\lambda_{s,n}^a\in[0,1]$ before observing $W_n$, and otherwise sets $\lambda_{s,n}^a=0$. Then the account wealth is updated via:
\begin{align*}
    G_{s,n+1}^a
    \coloneqq G_{s,n}^a(1+\lambda_{s,n}^aW_n)
\end{align*}
Since $W_n\geq-1$, wealth remains nonnegative.
The stake controls the gains from positive scores and the losses from negative scores. A good betting strategy can therefore help wealth grow rapidly when the certificate condition fails, allowing the monitor to raise an alarm sooner.

\paragraph{Mixing experts}
Different betting strategies may detect different violations
more quickly. Since the monitor does not know in advance
which strategy will work best, it combines their wealth accounts. Specifically, for each observation index $s$, the monitor combines the accounts in $\aA_s$ using fixed, nonnegative, $\aF_s$-measurable weights $\eta_{s,a}$ satisfying $\sum_{a\in\aA_s}\eta_{s,a}=1$. The combined wealth is $Q_{s,n}\coloneqq\sum_{a\in\aA_s}\eta_{s,a}G_{s,n}^a$ with $Q_{s,s}=1$.

A mixture may lose wealth before a violation begins. To launch fresh mixtures without earlier losses, the monitor
reserves part of its initial unit capital. 
The monitor starts with one unit of capital. Let $c_s$ denote
the amount allocated to the mixture launched at observation
$s$, with $c_s=0$ if no mixture is funded. Before launching
the mixture at index $n$, the monitor has already allocated
$C_n\coloneqq\sum_{s=0}^{n-1}c_s$ to earlier mixtures, leaving
$1-C_n$ in reserve. It allocates an amount
$c_n\in[0,1-C_n]$ from this reserve to the new mixture.
This amount may depend on the observations so far
($c_n$ is $\aF_n$-measurable).
Launching a mixture therefore transfers existing capital
rather than adding new capital.
The total monitor wealth is the remaining reserve plus
the current wealth of all funded mixtures:
\[
M_n\coloneqq 1-C_n+
\sum_{s=0}^{n-1}c_s
\sum_{a\in\aA_s}\eta_{s,a}
\prod_{k=s}^{n-1}(1+\lambda_{s,k}^aW_k),
\qquad M_0=1.
\]
Since $Q_{n,n}=1$, funding a new mixture does not change the total wealth.
For example, funding every $q\in\NN_{\geq1}$ observations can use $ c_{jq}=6/(\pi^2(j+1)^2)$ for $j\in\NN_0$ and $c_n=0$ otherwise. The allocations sum to one and leave reserve
$6/\pi^2\sum_{i=j+2}^{\infty}i^{-2}>0$ so there always remains fresh capital.

\paragraph{Probabilistic guarantees}
If the certificate condition holds almost surely throughout execution, Lemma~\ref{lem:score-validity} gives $\lambda_{s,n}^a\mu_n=\lambda_{s,n}^a\expe[W_n\mid\aF_n]\leq0$, and since the current wealth and the stake are determined by
the observations available before $W_n$ is observed, we obtain
\begin{align*}
    \expe[G_{s,n+1}^a\mid\aF_n]
    =G_{s,n}^a(1+\lambda_{s,n}^a\mu_n)
    \leq G_{s,n}^a.
\end{align*}
Thus, each expert account is a nonnegative supermartingale for any admissible stake strategy.
Combining these accounts using fixed mixture weights preserves this property for $Q_{s,n}$. Funding new mixtures from the remaining reserve likewise preserves it for $M_n$.
The threshold in~\eqref{eq:ville-bound} therefore applies to the total wealth regardless of the admissible stakes, mixture weights, and funding chosen.

\begin{theorem}
\label{thm:monitor-validity}
For every admissible choice of stakes, expert weights, and
funding, the monitor satisfies
\begin{align}
    \prob\left(
        \tau_\delta<\infty
        \ \text{and}\
        E_k\leq0\text{ for every }k<\tau_\delta
    \right)\leq\delta
    \label{eq:sampled-interval-validity},
\end{align}
and, if the expected-decrease condition holds for
the deployed process, then $ \prob(\tau_\delta<\infty)\leq\delta.$
\end{theorem}
Inequality~\eqref{eq:sampled-interval-validity} bounds the
probability of raising an alarm before the expected-decrease condition is violated, thus bounding the false-positive rate of the monitor, as desired. The proof is in Appendix~\ref{app:proofs}.

\paragraph{Detection power}
We also bound the probability that the monitor raises an alarm
within a block of observations during which the supermartingale
expected-decrease condition is persistently violated. Consider observation intervals
indexed by $k=s,\ldots,s+N-1$, and suppose that there exists $\varepsilon \in (0,1]$ such that $\expe[W_k\mid\aF_k]\geq\varepsilon$ and $W_k\leq1$ almost surely, for each $k$. Suppose that a betting strategy $a$, launched at index $r\leq s$,
has capital $\rho:=c_r\eta_{r,a}G_{r,s}^a>0$ at observation $s$
and uses the constant stake $\varepsilon/2$ over these intervals.
Then, for any $\gamma\in(0,1)$,
\begin{align*}
    \prob\left(\tau_\delta\leq s+N\mid\aF_s\right)
    \geq1-\gamma
    \quad\text{almost surely on }
    \left\{N\geq\frac{8}{\varepsilon^2}
    \log\frac{1}{\delta\rho\gamma}\right\}.
\end{align*}
Thus, if the violation persists for sufficiently many observation
intervals, the monitor raises an alarm by observation $s+N$
with conditional probability at least $1-\gamma$. The parameter $\varepsilon$ is a lower bound on the expected score at each observation. 
A larger value of $\varepsilon$ reduces the number of observation intervals
required to detect violation. A larger starting capital $\rho$
also reduces this number, because less wealth growth is needed
to reach the alarm threshold $1/\delta$.

\subsection{Weighted prediction strategies}
\label{sec:weighted-prediction}
We now give two concrete strategies for choosing the stake by predicting the next score from past scores. 
The better the monitor predicts the next score, the faster it can grow its capital.
The prediction is based on the fact that many realistic systems exhibit regularity in space and time.
Here, \emph{spatial regularity} means that nearby states have similar conditional expected scores, whereas \emph{temporal regularity} means that these expectations change slowly over time.
We consider a mixture over several experts,
where each expert guesses and weights past scores according to their spatial and temporal proximity to the current observation. If the guesses are accurate, this weighting can improve our estimate of the conditional mean of the next score.

\paragraph{Prediction weights}
For computational efficiency, we limit ourselves to a finite history, i.e., let $m\geq1$ be the number of scores retained in the history, and define the retained indices $I_n:=\{i\in\mathbb N_0:\max\{0,n-m\}\leq i<n\}.$
Let $a\in\mathcal{A}_s$ be an expert launched at observation index $s$.
The expert chooses a spatial scale $r^a_{\aX,s}$ and a temporal scale
$r^a_{\sT,s}$, each in $(0,\infty]$. We describe how the scales are obtained in Appendix C.
The scales serve as input to a spatial kernel $k_{\aX, s}^a$ and a temporal kernel $k_{\sT, s}^a$, respectively. Both kernels are nonincreasing, take values in $[0,1]$,
and equal one at zero. The kernels and scales are fixed when the expert
is launched. At each observation index $n\geq s$, $a$ assigns each retained
score $W_i$, $i\in I_n$, a prediction weight based on its starting
time $t_i$ and starting state $\rX_{t_i}$:
\begin{align}
\omega_{s,n,i}^a \coloneqq
k_{\sT,s}^a\left(\frac{t_n-t_i}{r_{\sT,s}^a}\right)
k_{\aX,s}^a\left(
    \frac{\|\rX_{t_n}-\rX_{t_i}\|}{r_{\aX,s}^a}
\right),
\qquad i\in I_n.
\label{eq:prediction-weights}
\end{align}
Intuitively, recent observations and observations starting near the current state
therefore receive greater weight. 
An infinite scale removes the corresponding dependence on distance,
using the convention $u/\infty=0$, e.g., temporal-only experts use $r_{\mathcal{X},s}^a=\infty$. Setting both scales to infinity gives uniform weights over the
retained history.

\paragraph{Betting strategies}
Next, we transform past scores and their prediction weights into stakes, using two concrete strategies when $X_{t_n}\in S$. A stake $\lambda$ multiplies the wealth by $1+\lambda W_n$,
so its \emph{log return} is $\log(1+\lambda W_n)$. Log returns add
across successive bets and therefore measure accumulated multiplicative
wealth growth. 
\emph{The log-return strategy} maximises the weighted historical log return:
\begin{align}
    \lambda_{s,n}^{a,\log}
    \coloneqq\min\operatorname*{arg\,max}_{0\leq\lambda\leq \overline{\lambda}}
    \sum_{i\in I_n}\omega_{s,n,i}^a\log(1+\lambda W_i) \quad \text{for $1/2 \leq \overline{\lambda} < 1$}
    \label{eq:log-return-stake} 
\end{align}

A second strategy is \emph{the moment strategy}, which maximises
a quadratic lower bound on the weighted historical log return.
Since $\log(1+\lambda w)\geq \lambda w-\lambda^2w^2$ for
$w\geq-1$ and $\lambda\in[0,1/2]$, this bound is
$\lambda S_{s,n}^a-\lambda^2U_{s,n}^a$, where
$S_{s,n}^a\coloneqq\sum_{i\in I_n}\omega_{s,n,i}^aW_i$
and $U_{s,n}^a\coloneqq\sum_{i\in I_n}\omega_{s,n,i}^aW_i^2$.
Its maximiser is
\begin{equation}
\lambda_{s,n}^{a,\mathrm{mom}}
\coloneqq \min\left\{\frac12,\frac{(S_{s,n}^a)_+}{2U_{s,n}^a}\right\},
\qquad (0/0\coloneqq0).
\label{eq:moment-stake}
\end{equation}
Thus positive weighted scores favour betting, whereas large
score magnitudes reduce the stake.

Both strategies choose zero when the weighted sum of the retained scores is nonpositive, including for an empty history. The moment strategy requires only two sums and bets no more than the log-return strategy for the same weights. 
At each launch, the monitor can include a range of experts $\mathcal{A}$ with different configurations of scales, kernels and strategies, to accommodate different violation patterns. Their accounts are then combined using the capital weights $\eta_{s,a}$ from~\Cref{sec:abstract-monitor}.

\section{Experimental Evaluation}
\label{sec:exp}
We evaluate on a wide range of benchmarks whether our monitoring framework can effectively detect violations of
supermartingale certificate conditions. We aim to address the following research questions. Can the monitors detect certificate violations under changed system dynamics (RQ1)? 
How do temporal and spatial scaling affect the power of monitoring (RQ2)?

\paragraph{Benchmarks} We consider three benchmark suites. 
The \emph{synthetic} suite consists of six families that isolate persistent, delayed, temporary,
and recurring local violations.
The \emph{control} suite comprises three systems which are variants of Building and Tanks
from PRoTECT~\citep{wooding2026protect}, together with an LQR system.
Their expected certificate increments can be computed
exactly, providing ground truth for evaluating detection.
The \emph{neural} suite uses 16 learned controller--certificate pairs
from LogRASM~\citep{BadingsKJJ25} to assess monitoring under modified
deployment dynamics.

\paragraph{Monitors}
We use three prediction schemes: uniform (U), which weights retained scores equally; temporal (T), which favours recent observations; and spatial (S), which favours observations from nearby states. On the \emph{synthetic} benchmarks, we compare moment (M) and log-return (L) strategies. The \emph{control} and \emph{neural} experiments use log-return and equal mixtures over five scales for T and S. Each monitor uses a false-alarm tolerance of $\delta=0.05$
and raises an alarm when its betting wealth reaches
$1/\delta=20$. Full configurations and implementation details are provided in the Appendix.

\paragraph{RQ1: Detection of certificate violations.} 
Our results show that the monitors detect certificate violations under changed dynamics, with spatial prediction being particularly
effective when violations are localised in the state space.
Table~\ref{tab:control-magnitude} illustrates how violation magnitude and exposure affect
detection. In Building, S-L detects every run, and its median
delay falls from 3043 to 732 observations as violations become
stronger and more frequent. The weaker, less frequent violations
in Tanks remain harder to detect, even with spatial prediction.
For LQR, we weaken the controller by reducing the applied
control input. With a 5\% reduction, violations occur throughout
each trajectory, yet each monitor detects fewer than half the
runs. With a 20\% reduction, all monitors detect every run. Overall this shows that 
even persistent violations can remain undetected when
their effect is too small to distinguish reliably from
stochastic fluctuations within the observed trajectory.

Table~\ref{tab:aggregate-results} summarises the mean eligible first-alarm percentages across a number of disturbed settings. On the \emph{control} benchmarks, S-L achieves 83.8\%, compared with 29.7\% for U-L and 28.2\% for T-L. Specifically in Building and Tanks, the certificate condition fails in
some regions even though it holds on average along the
trajectory. This suggests spatial prediction helps detect these local
failures by emphasising observations from nearby states. The \emph{neural} benchmarks show a similar advantage under expected nonincrease ($\epsilon=0$), where S-L raises alarms on 43.1\% of runs, compared with approximately 11\% for either baseline. The higher alarm rates for positive $\epsilon$ concern stronger decrease requirements tested on the same trajectories.

\begin{table*}[t]
    \centering
    \caption{Violation strength and detection in selected \emph{control}
settings, with 50 runs each. 
For each run, magnitude is the average of $\mu_n$ over violating observations, and exposure is the percentage of violating observations. Magnitude and exposure are reported as medians over the 50 runs. 
Detection entries give counts, with median delay in parentheses. Detection delay is measured in observations from violation onset
or first entry into the violating region. U-L, T-L, and S-L use uniform,
temporal, and spatial prediction, respectively. NR means fewer than $25$ runs alarm.}
    \small
    \setlength{\tabcolsep}{4pt}
    \begin{tabular}{@{}l|r|r|r|r|r@{}}
        \toprule
        & & & \multicolumn{3}{c}{Detections (median delay)} \\
        \cmidrule(lr){4-6}
        Benchmark \& Setting & Magnitude & Exposure (\%)
                & U-L & T-L & S-L \\
        \midrule
        Building: $8\%$ heating loss
            & 0.0151 & 24.9
            & 0 (NR) & 0 (NR) & 50 (3043) \\
        Building: $12\%$ heating loss
            & 0.0309 & 37.6
            & 0 (NR) & 0 (NR) & 50 (732) \\
        \midrule
        Tanks: burst probability $0.35$
            & 0.00561 & 5.5
            & 0 (NR) & 0 (NR) & 0 (NR) \\
        Tanks: burst probability $0.45$
            & 0.00615 & 11.0
            & 0 (NR) & 0 (NR) & 34 (28507) \\
        \midrule
        LQR: $5\%$ feedback loss
            & 0.00955 & 100.0
            & 19 (NR) & 12 (NR) & 16 (NR) \\
        LQR: $20\%$ feedback loss
            & 0.0383 & 100.0
            & 50 (2241) & 50 (1349) & 50 (2264) \\
        \bottomrule
    \end{tabular}
    \label{tab:control-magnitude}
\end{table*}

\begin{table*}[t]
    \centering
    \caption{Median detection delay in observations on the six \emph{synthetic} benchmark
families: persistent, delayed, temporary, direct recurrence,
stochastic ring, and hill. 
    XS--XL denote 
    extra-small through extra-large prediction scales (exact parameters given in Appendix). Temporal scales determine how far back in the history observations receive weight; spatial scales determine the neighbourhood radius around the current state.
    \emph{Equal} divides capital equally among the five scale experts. \emph{Uniform} weights all retained observations equally.
    NR means fewer than $25$ runs alarm.
    }
    \footnotesize
    \setlength{\tabcolsep}{2.5pt}
    \begin{tabular}{@{}lrrrrrr@{\hspace{10pt}}|lrrrrrr@{}}
        \toprule
        \multicolumn{7}{c}{Temporal prediction}
        & \multicolumn{7}{c}{Spatial prediction} \\
        \cmidrule(lr){1-7}\cmidrule(lr){8-14}
        Scale & Pers. & Del. & Temp. & Rec. & Ring & Hill
        & Scale & Pers. & Del. & Temp. & Rec. & Ring & Hill \\
        \midrule
        Uniform & 2126 & 3204 & NR & NR & 2360 & 588
        & Uniform & 2126 & 3204 & NR & NR & 2360 & 588 \\
         Equal & 2183 & 668 & 914 & 1102 & 664 & 492
        & Equal & 3351 & 2141 & NR & 571 & 654 & 482 \\
        \midrule
        XS & NR & 1565 & NR & 570 & 168 & 568
        & XS & NR & 3979 & NR & NR & NR & 670 \\
        S & 6719 & 756 & NR & 2110 & 1127 & 441
        & S & NR & 2977 & NR & NR & 655 & 594 \\
        M & 2442 & 575 & 766 & NR & 1713 & 476
        & M & 21051 & 3437 & NR & 1578 & 640 & 552 \\
        L & 2118 & 1174 & NR & NR & 2338 & 560
        & L & 7406 & 2275 & NR & 1040 & 1652 & 447 \\
        XL & 2127 & 2439 & NR & NR & 2347 & 587
        & XL & 2190 & 2379 & NR & 1576 & 2153 & 472 \\
        \bottomrule
    \end{tabular}
    \label{tab:synthetic-aggregate-delay}
\end{table*}

\paragraph{RQ2: Temporal vs. spatial scaling impact}
Our results show that uniform scaling, while sufficient for long persistent violations, struggles on traces with transient violations. Both temporal and spatial weighting are able to detect such behaviour. Spatial weighting, in particular, benefits from recurrent observations.
The correct scaling depends on the system, i.e., no uniformly best scaling is observed. 
However, mixing over scales can retain some of the benefits of the respective scales, being often only slightly worse than the best for the respective benchmark. Table~\ref{tab:synthetic-aggregate-delay} summarises the results. We can observe that the uniform monitor has a median detection delay of 2126 on persistent violations, on par with large spatial and temporal scales. By contrast, for temporary and recurrent violations (Rec. and Ring) its median detection delay is either larger or fewer than 50\% of the runs detect a fault. For example, on the spatial recurring ring benchmark, the medium spatial monitor has a delay of $640$, while the uniform monitor has a delay of $2360$; on the temporary fault benchmark only the medium temporal monitor alarmed among the individual scales in more than 50\% of the cases.

\begin{table}[t]
    \centering
    \caption{Percentage of runs in which each monitor raises an alarm,
averaged across settings (50 runs per setting). For \emph{neural}, $\epsilon=0$ tests expected nonincrease;
larger $\epsilon$ requires faster expected decrease.
}
    \small
    \setlength{\tabcolsep}{4pt}
    \begin{tabular}{@{}l|r|r|r|r@{}}
        \toprule
        Benchmark set & No. Settings & U-L(\%) & T-L(\%) & S-L(\%) \\
        \midrule
        Synthetic                  & 18 & 62.2 & 77.0 & 79.6 \\
        Control                    & 18 & 29.7 & 28.2 & 83.8 \\
        Neural, $\epsilon=0$       & 48 & 10.7 & 10.8 & 43.1 \\
        Neural, $\epsilon=0.0005$  & 48 & 71.9 & 71.7 & 77.8 \\
        Neural, $\epsilon=0.001$   & 48 & 76.1 & 76.3 & 81.0 \\
        \bottomrule
    \end{tabular}
    \label{tab:aggregate-results}
\end{table}

\section{Related work}

Supermartingale certificates have found broad applications in probabilistic program verification and stochastic (neural) control~\citep{BadingsKJJ25,ZikelicLVCH23,ChatterjeeNZ17,ChatterjeeGMZ22,TakisakaOUH21,MajumdarS25,KuraU26,PrajnaJP07,AnsaripourCHLZ23,XueZF24,AbateGR24,HenzingerMSZ25,DBLP:conf/lics/AbateGIR26}. Our work is closest to recent approaches for runtime monitoring of neural Lyapunov certificates in deterministic systems~\citep{YuZH25,HenzingerKY25}. Runtime monitoring is comparatively straightforward in the deterministic setting: a violation of the required decrease condition is directly observable from a single state transition. For supermartingale certificates, however, monitoring is fundamentally harder because the required decrease holds only in expectation and therefore cannot be assessed from a single transition. We overcome this challenge using betting-based sequential testing~\citep{ramdas2023game,waudby2024betting,TimansVNN25}, which we apply for the first time to supermartingale certificates for temporal properties. We further exploit temporal and spatial regularity in stochastic systems to make informed bets and improve detection power.

Runtime monitoring is an established technique for ensuring correctness of cyber-physical systems~\citep{BartocciDDFMNS18,BauerLS11}, providing a lightweight, black-box alternative to formal verification. Existing methods typically detect property violations directly from observed states, fundamentally restricting monitorable properties to safety~\citep{AlpernS85,PnueliZ06} and excluding properties such as reachability, stability, and general $\omega$-regular properties. Predictive monitoring~\citep{BortolussiCPSS19,LindemannQDP23,HenzingerKMYZ25} extends this scope by forecasting future states, but remains limited to finite time horizons. In contrast, by monitoring supermartingale certificates, our approach can detect violations of certificate conditions supporting general temporal properties over unbounded horizons. Related work~\citep{HindyLBKSP24arXiv,henzinger2025} uses martingale theory to detect distribution shifts or model misalignment in stochastic processes, but does not monitor the validity of certificates.

\section{Conclusion}
We introduced the first framework for runtime monitoring of supermartingale certificates for temporal properties that is applicable to both discrete- and continuous-time systems and to a broad range of supermartingale certificates. Our framework bounds the probability of ever raising a false alarm by a given threshold. Experimental evaluation demonstrates the practical effectiveness of our method for detecting supermartingale certificate violations at runtime in a range of stochastic control tasks. Interesting directions of future work include designing shielding and repair procedures that would use the alarms raised by our runtime monitoring framework.

\clearpage
\subsection*{AI use statement}
We used ChatGPT to assist with writing, literature search, theoretical development, and the design and implementation of experiments. We describe these uses below. For writing, we used these tools to draft and revise text, simplify explanations, shorten sections, and improve consistency in notation and terminology. We also used them to assist with LaTeX formatting, e.g., constructing tables. For literature search, we used these tools to identify relevant work and compare its assumptions and guarantees with ours. For theoretical development, we used these tools to suggest proof arguments, expand derivations, and identify possible gaps and counterexamples. They also helped clarify certain measure-theoretic arguments required in the continuous-time setting. For experiments, we used these tools to write basic code and generate plots. They supported the selection of system disturbances. They were used to draft the experimental setup in the appendix. 
The authors take full responsibility for the final content of the paper, including all AI-assisted text, mathematical claims and proofs, references, code, and experimental results.

\bibliographystyle{plainnat}
\bibliography{references}

\newpage
\appendix

\section{Continuous Time}
\label{app:sec:continuous}
\paragraph{Termination.}
For a drift condition $(\aS,l)$, we assume that the deployed process terminates at its first exit from $\aS$.
We retain the terminal state, so $\pX$ denotes the absorbed process.
More precisely,
\[
    \sigma\coloneqq\inf\{t\geq0:\rX_t\notin\aS\},
    \qquad \inf\varnothing\coloneqq\infty,
\]
is an $(\aG_t)$-stopping time, and, on $\{\sigma<\infty\}$,
$\rX_\sigma\notin\aS$ and $\rX_t=\rX_\sigma$ for every $t\geq\sigma$.
We do not require $\sigma$ to be finite.
For example, for a reach--avoid certificate, $\aS$ is the continuation region before reaching the target or an unsafe state.
This assumption concerns the actual first exit, including exits between observations.
This stopping-time requirement ensures that the process does not leave $\aS$ without our knowledge. Meaning, we are only concerned with the certificate and loss in $\aS$. Hence, we extend the allowed drift by zero outside $\aS$:
\begin{align}
    l_{\aS}(x)\coloneqq
    \begin{cases}
        l(x), & x\in\aS,\\
        0, & x\notin\aS.
    \end{cases}
    \label{eq:ct-terminal-drift}
\end{align}
Absorption ensures that an exit is detectable at the next observation because the terminal state is retained.
The proof of Lemma~\ref{lem:ct-terminal-increments} uses this to extend the regional drift inequality to all visited states.

\paragraph{Continuous-time assumptions.}
In continuous time, the expected drift condition bounds the instantaneous rate of change of the certificate. To obtain an expected bound on the certificate change between observations, we impose the following assumptions. 
As in the preliminaries, $\pX$ is a right-continuous, time-homogeneous Markov process, and $(\aG_t)$ is complete and right-continuous.
Right-continuity and adaptation ensure the joint measurability used in the Tonelli argument in the proof of Lemma~\ref{lem:ct-terminal-increments}.

\begin{itemize}
    \item \emph{Generator:}
    There is a finite measurable function $\mathcal Af$ such that, for every fixed $t\geq0$, almost surely,
    \[
        \mathcal Af(\rX_t)
        =\lim_{u\downarrow0}
        \frac{\expe[f(\rX_{t+u})\mid\rX_t]-f(\rX_t)}{u}.
    \]
    This is the generator drift of the absorbed process, with $\mathcal Af=0$ outside $\aS$.
    The identity is used in the proof of Lemma~\ref{lem:ct-terminal-increments} to express the instantaneous certificate condition as
    $\mathcal Af(\rX_r)\leq l(\rX_r)$ on $\{\rX_r\in\aS\}$.

    \item \emph{Integrability:}
    For every finite $T$, we assume
    \begin{align}
        \expe\left[
            |f(\rX_0)|
            +\int_0^T
            \bigl(|\mathcal Af(\rX_r)|+|l_{\aS}(\rX_r)|\bigr)
            \,\mathrm{d}r
        \right]<\infty.
        \label{eq:ct-terminal-integrability}
    \end{align}
    Together with integrability of $N_t^f$ below, this establishes integrability of $f(\rX_t)$ and $L_n$ in the opening argument of the proof of Lemma~\ref{lem:ct-terminal-increments}.
    It also ensures that the conditional expectation of the integrated drift difference in the final step is finite.

    \item \emph{Martingale:}
    The process $(N_t^f)_{t\geq0}$ defined by
    \begin{align}
        N_t^f\coloneqq
        f(\rX_t)-f(\rX_0)
        -\int_0^t\mathcal Af(\rX_r)\,\mathrm{d}r
        \label{eq:ct-terminal-martingale}
    \end{align}
    is a right-continuous $(\aG_t)$-martingale.
    Thus, subtracting the accumulated actual drift leaves a process with zero conditional expected change.
    The proof of Lemma~\ref{lem:ct-terminal-increments} uses this property to obtain
    $\expe[N_{t_{n+1}}^f-N_{t_n}^f\mid\aG_{t_n}]=0$,
    leaving only the integrated drift difference in the expression for $E_n$.
\end{itemize}
For the martingale formulation of extended generators, see \citet[Definition~4.1 and Proposition~4.1]{hirsch2012temporally}.
We additionally require the displayed infinitesimal limit to exist.
The certificate condition $\mathcal Af\leq l$ on $\aS$ is imposed separately.
The usual filtration conditions and right-continuity of $N^f$ are retained as regularity assumptions; the fixed-time proof of Lemma~\ref{lem:ct-terminal-increments} does not use them separately.

\paragraph{Observations.}
We observe the absorbed process at fixed times $t_n=nh$.
The observation filtration satisfies $\aF_n\subseteq\aG_{t_n}$, and $\rX_{t_n}$ is $\aF_n$-measurable.
The proof of Lemma~\ref{lem:ct-terminal-increments} applies the martingale identity at these fixed times and uses the filtration inclusion to condition further on $\aF_n$.
Measurability of the sampled state makes the loss bound below predictable, as required when taking $1/B_n$ outside the conditional expectation in the proof of Lemma~\ref{lem:score-validity}.

The accumulated allowed drift is
\begin{align}
    L_n\coloneqq
    \int_{t_n}^{t_{n+1}}l_{\aS}(\rX_r)\,\mathrm{d}r
    =
    \int_{t_n\wedge\sigma}^{t_{n+1}\wedge\sigma}
        l_{\aS}(\rX_r)\,\mathrm{d}r.
    \label{eq:ct-terminal-allowance}
\end{align}
Thus, the allowed drift stops accumulating at termination.
The interval containing termination is scored using the retained terminal state.
On later intervals, we set $\overline L_n=0$, giving $D_n=W_n=0$.
Termination itself does not raise a statistical alarm.

\paragraph{Bounds at termination.}
Suppose that $l(x)\leq c$ on $\aS$ and $f(x)\geq f_{\min}$ on
all reachable states, including terminal states.
Write $c_+\coloneqq\max\{c,0\}$.
If the termination time is available when it occurs, define
\begin{align*}
    \Delta_n\coloneqq
        (t_{n+1}\wedge\sigma)-(t_n\wedge\sigma),
    \qquad
    \overline L_n\coloneqq c\Delta_n.
\end{align*}
Then $L_n\leq c\Delta_n$ for either sign of $c$.
If only sampled states are available, an upper bound is
\begin{align}
    \overline L_n\coloneqq
    \begin{cases}
        ch, & \rX_{t_{n+1}}\in\aS,\\
        c_+h, & \rX_{t_n}\in\aS,\ \rX_{t_{n+1}}\notin\aS,\\
        0, & \rX_{t_n}\notin\aS.
    \end{cases}
    \label{eq:ct-terminal-observable-bound}
\end{align}
Indeed, survival until $t_{n+1}$ ensures that the whole interval was
spent in $\aS$, while a terminating interval has duration
$0\leq\Delta_n\leq h$ and $L_n\leq c\Delta_n\leq c_+h$.
If the process has already terminated, $L_n=0$.
In both cases, the predictable bound
\begin{align}
    B_n\coloneqq
    \max\{b_{\min},\,f(\rX_{t_n})+c_+h-f_{\min}\}
    \label{eq:ct-terminal-loss-bound}
\end{align}
satisfies $D_n\geq-B_n$, because $\overline L_n\leq c_+h$.
For a negative constant drift bound, using $ch$ on the terminating
interval would generally underestimate the accumulated allowance.
The piecewise bound above avoids this without requiring the exact
termination time.

\begin{lemma}[Continuous-time expected increments]
\label{lem:ct-terminal-increments}
Under the continuous-time assumptions above, $f(\rX_t)$ is integrable
for every finite $t$. If the $(\aS,l)$-expected drift condition holds,
then, for every $n\in\NN_0$,
\begin{align}
    E_n\coloneqq
    \expe\left[
        f(\rX_{t_{n+1}})-f(\rX_{t_n})-L_n
        \,\middle|\,\aF_n
    \right]\leq0
    \qquad\text{almost surely}.
    \label{eq:ct-terminal-increment-validity}
\end{align}
\end{lemma}

\begin{proof}
Fix an observation index $n\in\NN_0$.

\emph{Step 1: We check integrability.}
By the martingale decomposition in \eqref{eq:ct-terminal-martingale},
\[
    f(\rX_t)
    =f(\rX_0)+N_t^f
     +\int_0^t\mathcal Af(\rX_r)\,\mathrm{d}r.
\]
Taking absolute values and expectations gives
\[
    \expe[|f(\rX_t)|]
    \leq
    \expe[|f(\rX_0)|]
    +\expe[|N_t^f|]
    +\expe\left[\int_0^t|\mathcal Af(\rX_r)|\,\mathrm{d}r\right]
    <\infty.
\]
The first and third terms are finite by
\eqref{eq:ct-terminal-integrability}.
The second is finite because $N^f$ is a martingale.
The same integrability assumption gives
\[
    \expe[|L_n|]
    \leq
    \expe\left[
        \int_{t_n}^{t_{n+1}}|l_{\aS}(\rX_r)|\,\mathrm{d}r
    \right]<\infty.
\]
Thus, the certificate values and allowance defining $E_n$ are integrable.

\emph{Step 2: We bound the drift difference at each fixed time.}
Fix $r\geq0$.
If $\rX_r\in\aS$, the certificate condition gives
\[
    \mathcal Af(\rX_r)\leq l(\rX_r)=l_{\aS}(\rX_r)
    \qquad\text{almost surely}.
\]
If $\rX_r\notin\aS$, the process is absorbed.
Its generator drift is therefore zero, and $l_{\aS}(\rX_r)=0$
by definition.
Combining the two cases yields
\[
    \mathcal Af(\rX_r)-l_{\aS}(\rX_r)\leq0
    \qquad\text{almost surely}.
\]

\emph{Step 3: We pass from fixed times to the whole interval.}
Write
\[
    J_n\coloneqq
    \int_{t_n}^{t_{n+1}}
        \bigl(\mathcal Af(\rX_r)-l_{\aS}(\rX_r)\bigr)
        \,\mathrm{d}r.
\]
This integral is integrable because
\[
    \expe[|J_n|]
    \leq
    \expe\left[
        \int_{t_n}^{t_{n+1}}
        \bigl(|\mathcal Af(\rX_r)|+|l_{\aS}(\rX_r)|\bigr)
        \,\mathrm{d}r
    \right]<\infty.
\]
We next show that $J_n\leq0$ almost surely.
Step~2 shows that, for each fixed time $r$, the drift difference is nonpositive with probability one. We now apply Tonelli's theorem to deduce that its integral over the observation interval is also nonpositive with probability one. For $(z)_+\coloneqq\max\{z,0\}$, joint measurability and nonnegativity give
\begin{align*}
    &\expe\left[
        \int_{t_n}^{t_{n+1}}
        \bigl(\mathcal Af(\rX_r)-l_{\aS}(\rX_r)\bigr)_+
        \,\mathrm{d}r
    \right] =
    \int_{t_n}^{t_{n+1}}
        \expe\left[
            \bigl(\mathcal Af(\rX_r)-l_{\aS}(\rX_r)\bigr)_+
        \right]
        \,\mathrm{d}r  =0.
\end{align*}
The last equality follows from Step~2 at each fixed $r$.
A nonnegative random variable with expectation zero is zero almost surely.
Consequently,
\[
    J_n
    \leq
    \int_{t_n}^{t_{n+1}}
        \bigl(\mathcal Af(\rX_r)-l_{\aS}(\rX_r)\bigr)_+
        \,\mathrm{d}r =0
    \qquad\text{almost surely}.
\]

\emph{Step 4: We decompose the corrected certificate increment.}
Subtracting \eqref{eq:ct-terminal-martingale} at $t_n$ from the
same identity at $t_{n+1}$ gives
\[
    f(\rX_{t_{n+1}})-f(\rX_{t_n})
    =N_{t_{n+1}}^f-N_{t_n}^f
     +\int_{t_n}^{t_{n+1}}\mathcal Af(\rX_r)\,\mathrm{d}r.
\]
Subtracting $L_n$ then yields
\[
    f(\rX_{t_{n+1}})-f(\rX_{t_n})-L_n
    =N_{t_{n+1}}^f-N_{t_n}^f+J_n.
\]
Thus, the corrected increment is the sum of a martingale increment
and the nonpositive accumulated drift difference.

\emph{Step 5: We take conditional expectations.}
The martingale property gives
\[
    \expe[N_{t_{n+1}}^f-N_{t_n}^f\mid\aF_{t_n}]
    =N_{t_n}^f-N_{t_n}^f=0.
\]
Since $\aF_n\subseteq\aF_{t_n}$, the tower property gives
\begin{align*}
    \expe[N_{t_{n+1}}^f-N_{t_n}^f\mid\aF_n]
    &=
    \expe\left[
        \expe[N_{t_{n+1}}^f-N_{t_n}^f\mid\aF_{t_n}]
        \,\middle|\,\aF_n
    \right]\\
    &=0.
\end{align*}
Applying conditional expectation to the decomposition in Step~4,
we therefore obtain
\begin{align*}
    E_n
    &=\expe\left[
        f(\rX_{t_{n+1}})-f(\rX_{t_n})-L_n
        \,\middle|\,\aF_n
    \right]\\
    &=\expe[N_{t_{n+1}}^f-N_{t_n}^f\mid\aF_n]
      +\expe[J_n\mid\aF_n]\\
    &=\expe[J_n\mid\aF_n]\leq0.
\end{align*}
The final inequality follows from $J_n\leq0$ almost surely and
monotonicity of conditional expectation.
\end{proof}

\section{Validity}
\label{app:proofs}

\paragraph{Score decrease.}
We show the score decrease condition. Intuitively, the score compares the observed certificate change with an upper bound on the allowed drift. Subtracting this upper bound can only reduce the corrected increment. Normalising by a positive bound chosen before the next observation then preserves the sign of its conditional mean.

\begin{proof}[Proof of Lemma~\ref{lem:score-validity}]
The integrability assumptions imply that $D_n$ is integrable.
Since $B_n\geq b_{\min}>0$,
$|W_n|\leq |D_n|/b_{\min}$, so $W_n$ is also integrable.
The loss bound gives $W_n\geq-1$.
Adding and subtracting $L_n$ and using the
$\aF_n$-measurability of $B_n$ gives
\begin{align*}
    \mu_n
    &=\frac{\expe[D_n\mid\aF_n]}{B_n}=\frac{E_n-\expe[\overline L_n-L_n\mid\aF_n]}{B_n}
    \leq\frac{E_n}{B_n}.
\end{align*}
The inequality follows from $\overline L_n\geq L_n$.
In continuous time, Lemma~\ref{lem:ct-terminal-increments} gives
$E_n\leq0$ under the certificate condition.
In discrete time, $h=1$ gives $t_{n+1}=t_n+1$.
On $\{\rX_{t_n}\in\aS\}$, we have $L_n=l(\rX_{t_n})$.
The Markov property and the one-step expected-drift condition therefore give
\begin{align*}
    E_n
    &=\expe[
        f(\rX_{t_{n+1}})-f(\rX_{t_n})
        \mid\aF_n
    ]-l(\rX_{t_n})\\
    &=\expe[
        f(\rX_{t_n+1})-f(\rX_{t_n})
        \mid\rX_{t_n}
    ]-l(\rX_{t_n})
    \leq0
\end{align*}
almost surely on $\{\rX_{t_n}\in\aS\}$.
Consequently, the certificate condition implies $\mu_n\leq0$
for all continuous-time observations and, in discrete time,
on $\{\rX_{t_n}\in\aS\}$.
\end{proof}

\paragraph{Statistical guarantee.}
For the proof, we freeze the wealth just before the first observation interval with positive conditional expected corrected drift.
Before freezing, each bet has nonpositive expected gain;
afterwards, the wealth stays constant.
The frozen wealth is therefore a supermartingale, so Ville's
inequality applies.

\begin{proof}[Proof of Theorem~\ref{thm:monitor-validity}]
We first construct an auxiliary monitor whose scores always have
nonpositive conditional means.
We then apply Ville's inequality to its wealth and show that every
alarm covered by the historical guarantee also crosses the auxiliary
wealth threshold.

\emph{Step 1: We replace scores from the first violation onwards.}
We define the first index at which the conditional expected
corrected increment is positive, i.e., $\kappa\coloneqq\inf\{n\in\NN_0:E_n>0\}$, with $\inf\varnothing\coloneqq\infty$ and freeze the process at that time $ \widetilde W_n\coloneqq\mathbf1\{n<\kappa\}W_n.$ The auxiliary scores agree with the actual scores before $\kappa$ and are zero from $\kappa$ onwards.
For every $k\leq n$, the random variable $E_k$ is
$\aF_k$-measurable and hence $\aF_n$-measurable.
Consequently,
\[
    \{n<\kappa\}
    =\bigcap_{k=0}^{n}\{E_k\leq0\}\in\aF_n.
\]
The indicator multiplying $W_n$ is therefore measurable before
that score is observed.
We may take it outside the conditional expectation:
\[
    \expe[\widetilde W_n\mid\aF_n]
    =\mathbf1\{n<\kappa\}\expe[W_n\mid\aF_n]
    =\mathbf1\{n<\kappa\}\mu_n.
\]
On $\{n<\kappa\}$, score validity gives
$\mu_n\leq E_n/B_n\leq0$.
On its complement, the displayed expression is zero.
Hence $\expe[\widetilde W_n\mid\aF_n]\leq0$ almost surely. 
Moreover, $|\widetilde W_n|\leq|W_n|$ and
$\widetilde W_n\geq-1$, so the auxiliary scores remain integrable
and satisfy the same lower bound.

\emph{Step 2: Each auxiliary account is a nonnegative supermartingale.}
For each expert $a$ launched at $s$, define
\[
    \widetilde G_{s,s}^a=1,
    \qquad
    \widetilde G_{s,n+1}^a
    =\widetilde G_{s,n}^a
       (1+\lambda_{s,n}^a\widetilde W_n).
\]
We use exactly the stakes chosen by the actual monitor.
These stakes remain $\aF_n$-measurable, because they are still
computed from the actual observation history.

Since $0\leq\lambda_{s,n}^a\leq1$ and $\widetilde W_n\geq-1$, it follows that $  1+\lambda_{s,n}^a\widetilde W_n
    \geq1-\lambda_{s,n}^a\geq0.$
Every account therefore remains nonnegative.
Both its current wealth and its next stake are $\aF_n$-measurable,
so Step~1 gives
\begin{align*}
    \expe[\widetilde G_{s,n+1}^a\mid\aF_n]
    &=
    \widetilde G_{s,n}^a
    \expe[1+\lambda_{s,n}^a\widetilde W_n\mid\aF_n] =
    \widetilde G_{s,n}^a
    \bigl(1+\lambda_{s,n}^a
    \expe[\widetilde W_n\mid\aF_n]\bigr) \leq\widetilde G_{s,n}^a.
\end{align*}

This calculation also establishes integrability.
Because the next account value is nonnegative, its conditional
expectation is defined even before integrability is established,
although it could initially be infinite.
Starting from $\widetilde G_{s,s}^a=1$, the displayed inequality
and the tower property give inductively
\[
    \expe[\widetilde G_{s,n+1}^a]
    \leq\expe[\widetilde G_{s,n}^a]\leq1.
\]
Thus every account value has finite expectation, and
$(\widetilde G_{s,n}^a)_{n\geq s}$ is a nonnegative supermartingale.

\emph{Step 3: The auxiliary mixtures are also supermartingales.}
Use the same mixture weights as the actual monitor and set
\[
    \widetilde Q_{s,n}
    \coloneqq\sum_{a\in\aA_s}\eta_{s,a}\widetilde G_{s,n}^a.
\]
The weights are nonnegative and sum to one, so
$\widetilde Q_{s,n}\geq0$ and $\widetilde Q_{s,s}=1$.
They are chosen at launch and are therefore
$\aF_s$-measurable, hence $\aF_n$-measurable for $n\geq s$.
Using Step~2,
\begin{align*}
    \expe[\widetilde Q_{s,n+1}\mid\aF_n]
    &=
    \sum_{a\in\aA_s}\eta_{s,a}
        \expe[\widetilde G_{s,n+1}^a\mid\aF_n]\\
    &\leq
    \sum_{a\in\aA_s}\eta_{s,a}\widetilde G_{s,n}^a
    =\widetilde Q_{s,n}.
\end{align*}
As above, nonnegativity and the initial value
$\widetilde Q_{s,s}=1$ give integrability by induction.
Thus each mixture is a nonnegative supermartingale.
No independence between its accounts is needed.

\emph{Step 4: Funding preserves the supermartingale property.}
Use the actual funding amounts $c_s$ and define
\[
    C_n\coloneqq\sum_{s=0}^{n-1}c_s,
    \qquad
    \widetilde M_n
    \coloneqq1-C_n+\sum_{s=0}^{n-1}c_s\widetilde Q_{s,n}.
\]
The funding constraint $0\leq c_n\leq1-C_n$ ensures
$0\leq C_n\leq1$.
Hence the reserve and every funded contribution are nonnegative,
giving $\widetilde M_n\geq0$.
The sum contains finitely many terms at each $n$.
Since $c_s\leq1$ and each mixture is integrable,
$\widetilde M_n$ is integrable as well.
Initially, no accounts have been funded, so $\widetilde M_0=1$.
At index $n$, funding removes $c_n$ from the reserve and assigns
it to a new mixture with unit initial wealth.
Thus
\[
    \widetilde M_{n+1}
    =1-C_n-c_n
     +\sum_{s=0}^{n-1}c_s\widetilde Q_{s,n+1}
     +c_n\widetilde Q_{n,n+1}.
\]
All funding amounts appearing here are $\aF_n$-measurable.
Taking conditional expectations and applying Step~3 gives
\begin{align*}
    \expe[\widetilde M_{n+1}\mid\aF_n]
    &=
    1-C_n-c_n
    +\sum_{s=0}^{n-1}c_s
        \expe[\widetilde Q_{s,n+1}\mid\aF_n]
    +c_n\expe[\widetilde Q_{n,n+1}\mid\aF_n]\\
    &\leq
    1-C_n-c_n
    +\sum_{s=0}^{n-1}c_s\widetilde Q_{s,n}
    +c_n\widetilde Q_{n,n}\\
    &=
    1-C_n+\sum_{s=0}^{n-1}c_s\widetilde Q_{s,n}
    =\widetilde M_n.
\end{align*}
The last equality uses $\widetilde Q_{n,n}=1$:
the capital assigned to the new mixture exactly matches the
amount removed from the reserve.
Therefore $(\widetilde M_n)$ is a nonnegative supermartingale.

\emph{Step 5: We bound threshold crossings of the auxiliary wealth.}
Since $\widetilde M_0=1$, Ville's inequality gives
\[
    \prob\left(
        \exists n\in\NN_0:\widetilde M_n\geq1/\delta
    \right)\leq\delta.
\]

\emph{Step 6: We transfer this bound to the actual monitor.}
Consider an outcome on which $\tau_\delta<\infty$ and
$E_k\leq0$ for every $k<\tau_\delta$.
By definition of $\kappa$, this implies
$\kappa\geq\tau_\delta$.
Consequently,
\[
    \widetilde W_k=W_k
    \qquad\text{for every }k<\tau_\delta.
\]
These are precisely the scores used to compute the wealth at
$\tau_\delta$.
In particular, equality $\kappa=\tau_\delta$ causes no problem:
the score with index $\tau_\delta$ has not yet contributed to
that wealth.

The actual and auxiliary accounts have the same initial values
and use the same stakes.
Since their scores agree through index $\tau_\delta-1$, their
account values agree at $\tau_\delta$.
Their mixture weights, funding amounts, and reserves also agree.
It follows that
\[
    \widetilde M_{\tau_\delta}
    =M_{\tau_\delta}\geq1/\delta.
\]
We have therefore established the event inclusion
\[
    \left\{
        \tau_\delta<\infty,\;
        E_k\leq0\text{ for every }k<\tau_\delta
    \right\}
    \subseteq
    \left\{
        \exists n\in\NN_0:\widetilde M_n\geq1/\delta
    \right\}.
\]
Applying the bound from Step~5 yields
\[
    \prob\left(
        \tau_\delta<\infty
        \ \text{and}\ E_k\leq0
        \text{ for every }k<\tau_\delta
    \right)\leq\delta.
\]

\emph{Step 7: We obtain the unconditional false-alarm bound.}
Suppose that the certificate condition holds.
Score validity gives $\mu_n\leq0$ almost surely on
$\{\rX_{t_n}\in\aS\}$.
Outside $\aS$, every expert chooses stake zero.
Consequently, for every expert and every $n\geq s$,
\[
    \lambda_{s,n}^a\mu_n\leq0
    \qquad\text{almost surely}.
\]
We can therefore apply the account argument directly to the
actual scores:
\begin{align*}
    \expe[G_{s,n+1}^a\mid\aF_n]
    &=G_{s,n}^a
      \bigl(1+\lambda_{s,n}^a
          \expe[W_n\mid\aF_n]\bigr)=G_{s,n}^a(1+\lambda_{s,n}^a\mu_n)
    \leq G_{s,n}^a.
\end{align*}
The nonnegativity and integrability arguments from Step~2 apply
unchanged.
The mixture and funding arguments from Steps~3 and~4 then show
that the actual wealth $(M_n)$ is a nonnegative supermartingale
with $M_0=1$.
Ville's inequality therefore gives
\[
    \prob(\tau_\delta<\infty)
    =\prob\left(\exists n\in\NN_0:M_n\geq1/\delta\right)
    \leq\delta.
\]

\end{proof}

\paragraph{Monitoring several conditions.}
For a certificate with several drift conditions, each component
monitor must satisfy the score assumptions for its own condition.
The termination assumption applies separately to the region used
in each continuous-time condition.
In particular, one physical termination time suffices when the
conditions share a continuation region.
If their regions differ, stopping at the target or an unsafe state
does not by itself justify every component: each component must
have access to the process stopped at its own first exit, including
its exit state, or use another valid interval construction.

\begin{corollary}[Composition]
\label{cor:certificate-composition}
Suppose that a certificate has $J$ expected-drift conditions.
For each condition $j$, run a monitor satisfying
Theorem~\ref{thm:monitor-validity} at level $\delta_j>0$, where
$\sum_{j=1}^J\delta_j\leq\delta$.
Let $\tau$ be the first index at which any component monitor or
sampled domain-bound check raises an alarm.
If all certificate conditions hold for the deployed process,
then $\prob(\tau<\infty)\leq\delta$.
\end{corollary}

\begin{proof}
Write $\tau_j$ for the first alarm index of component $j$.
A valid domain-bound condition holds at every fixed sampled time
almost surely. Since there are countably many samples and finitely
many such conditions, their checks raise no alarm almost surely.
The component guarantees and the union bound therefore give
\begin{align*}
    \prob(\tau<\infty)
    \leq\sum_{j=1}^J\prob(\tau_j<\infty)
    \leq\sum_{j=1}^J\delta_j
    \leq\delta.
\end{align*}
Independence is not required.
\end{proof}

\section{Power}
\paragraph{Weighted stakes.}
The two stake strategies use the same weighted history to decide how strongly to bet. The log-return strategy directly maximises the historical log return.
The moment strategy instead maximises a quadratic lower bound, which requires only the weighted sum of the scores and their squares.
The following lemma justifies both constructions and explains how these quantities relate to the expected return of the next bet.
For the following lemma we fix an expert $a\in\mathcal A_s$, an observation index $n\geq s$, and a cap $1/2\leq\overline\lambda<1$. Moreover, for $ 0\leq\lambda\leq\overline\lambda$, define:
\begin{align*}
    \omega_i\coloneqq\omega_{s,n,i}^a,
    \quad 
    S\coloneqq\sum_{i\in I_n}\omega_iW_i,
    \quad
    U\coloneqq\sum_{i\in I_n}\omega_iW_i^2, \quad  H(\lambda)
    \coloneqq\sum_{i\in I_n}\omega_i\log(1+\lambda W_i).
\end{align*}
\begin{lemma}[Weighted stake strategies]
\label{lem:weighted-stake-strategies}
The following statements hold on $\{\rX_{t_n} \in \aS\}$.
\begin{enumerate}
    \item
    The log-return and moment stakes are $\mathcal F_n$-measurable, both are zero exactly when $S\leq0$, and $ 0 \leq\lambda_{s,n}^{a,\mathrm{mom}}      \leq\lambda_{s,n}^{a,\log} \leq\overline\lambda.$
    \item
    For every $\lambda\in[0,1/2]$, the moment stake maximises the lower bound $ H(\lambda)\geq\lambda S-\lambda^2U$ over $[0,1/2]$.

    \item
    For any $\mathcal F_n$-measurable stake $\lambda_n\in[0,1/2]$, if $ q_n\coloneqq \mathbb E[W_n^2\mid\mathcal F_n]<\infty$ a.s.\ then
    \begin{align}
        \mathbb E[\log(1+\lambda_nW_n)\mid\mathcal F_n]
        \geq\lambda_n\mu_n-\lambda_n^2q_n
        \qquad\text{almost surely}.
        \label{eq:conditional-log-growth}
    \end{align}
\end{enumerate}
\end{lemma}

\begin{proof}

We first establish the quadratic lower bound.
For $z\geq-1/2$, define $\psi(z)\coloneqq\log(1+z)-z+z^2$. We take its derivative $ \psi'(z)=\frac{z(1+2z)}{1+z}$ and observe that $\psi'(z)\leq0$ on $[-1/2,0]$, because the factor $z$ is nonpositive, while $1+2z$ is nonnegative and $1+z$ is positive. Moreover, on $[0,\infty)$, all factors are nonnegative, so $\psi'(z)\geq0$. Thus, $\psi$ attains its minimum at zero.
Since $\psi(0)=0$, substituting $z=\lambda w$ implies that
\begin{align}
    \log(1+\lambda w)
    \geq\lambda w-\lambda^2w^2
    \qquad
    \text{for }w\geq-1,\quad 0\leq\lambda\leq\tfrac12.
    \label{eq:quadratic-log-lower-bound}
\end{align}
Multiplying by the nonnegative weights and summing yields $  H(\lambda)\geq\lambda S-\lambda^2U$.

The linear term rewards a positive weighted score sum, while the quadratic term penalises large stakes.
If $U>0$, the derivative of this lower bound is $S-2\lambda U$, so its maximiser over $[0,1/2]$ is
\begin{align*}
    \lambda_{s,n}^{a,\mathrm{mom}}
    =\min\left\{\frac12,\frac{S_+}{2U}\right\}.
\end{align*}
If $U=0$, every score with positive weight is zero.
Hence, $S=0$, the lower bound is constant, and the strategy chooses zero.

We next consider the log-return strategy.
Since $W_i\geq-1$ and $\overline\lambda<1$, the logarithms are well-defined throughout the optimisation interval.
The objective is continuous on a compact interval, so a maximiser exists.
Its derivatives are
\begin{align*}
    H'(\lambda)
    &=\sum_{i\in I_n}\frac{\omega_iW_i}{1+\lambda W_i},
    &
    H''(\lambda)
    &=-\sum_{i\in I_n}
    \frac{\omega_iW_i^2}{(1+\lambda W_i)^2}.
\end{align*}
Consequently, $H$ is concave, and strictly concave whenever $\sum_{i\in I_n}\omega_iW_i^2>0$. The derivative $H'(0)=\sum_{i\in I_n}\omega_iW_i$ describes whether increasing the stake from zero initially improves the historical return.
If $S\leq0$, concavity implies that zero is the smallest maximiser.
If $S>0$, the maximiser is positive.
It equals $\overline\lambda$ when $H'(\overline\lambda)\geq0$; otherwise, it is the unique root of $H'$ in $(0,\overline\lambda)$.
This gives the endpoint checks and bisection procedure used in the implementation.
All inputs are $\mathcal F_n$-measurable, as are the bisection iterates and their limit.
The explicit moment stake is likewise measurable.

To compare the stakes, suppose $S>0$.
For $\lambda\in[0,1/2]$, we have $1+\lambda W_i\geq1/2$, and therefore
\begin{align*}
    H'(\lambda)
    &=S-\lambda\sum_{i\in I_n}
      \frac{\omega_iW_i^2}{1+\lambda W_i} \geq S-2\lambda U.
\end{align*}
At the moment stake, the final expression is nonnegative.
Thus, the historical log return has not started decreasing at that stake, so its maximiser cannot lie below it.
If $S\leq0$, both stakes are zero.

Finally, the pointwise logarithmic bound also applies to the next score.
Taking conditional expectations and using predictability gives
\begin{align*}
    \mathbb E[\log(1+\lambda_nW_n)\mid\mathcal F_n]
    \geq\lambda_n\mu_n-\lambda_n^2q_n.
\end{align*}
The logarithm is integrable because
\begin{align*}
    \bigl|\log(1+\lambda_nW_n)\bigr|
    \leq\log2+\tfrac12(W_n)_+,
\end{align*}
and $W_n$ is integrable under the standing assumptions.
\end{proof}
The conditional bound explains why the moment strategy uses both the scores and their squares.
A positive conditional mean favours betting, while a large second moment makes a large stake less attractive.
In particular, the expected log return is positive whenever
\begin{align*}
    \lambda_n>0
    \qquad\text{and}\qquad
    \mu_n>\lambda_nq_n.
\end{align*}
Thus, sufficiently small positive stakes ensure positive expected log return when the conditional mean is positive and the conditional second moment is finite. The log-return strategy achieves at least as much weighted historical log return as the moment strategy, but neither is necessarily better for the next score, e.g., overcommitting to a wrong outcome may decrease wealth drastically.

\paragraph{Locality.}
We choose spatial and temporal scales using guesses about the regularity of the conditional expected scores.
Fix $\theta\in(0,1)$ and an expert $a\in\aA_s$ launched at observation index $s$, which aims to detect a conditional expected score of at least $\mu_{\star,s}^a>0$.
The expert guesses spatial and temporal Lipschitz constants
$\alpha_{\aX,s}^a>0$ and $\alpha_{\sT,s}^a>0$.
For the spatial regularity argument, suppose that
$\mu_n=\expe[W_n\mid\aF_n]=g_h(\rX_{t_n})$ almost surely for every $n$,
where $g_h\colon\aX\to\RN$ gives the expected score over an observation interval of length $h$.
The regularity guesses are
\begin{equation}
\label{eq:lipschitz}
\begin{aligned}
|g_h(y)-g_h(x)|
  &\leq\alpha_{\aX,s}^a\lVert x-y\rVert,
  &&\text{for all }x,y\in\aX,\\
|\mu_n-\mu_i|
  &\leq\alpha_{\sT,s}^a|t_n-t_i|,
  &&\text{almost surely, for all }0\leq i<n.
\end{aligned}
\end{equation}
Using $\theta\mu_{\star,s}^a$ as the tolerated difference between past and current conditional means, we define
\[
    r_{\aX,s}^a\coloneqq
    \frac{\theta\mu_{\star,s}^a}{\alpha_{\aX,s}^a},
    \qquad
    r_{\sT,s}^a\coloneqq
    \frac{\theta\mu_{\star,s}^a}{\alpha_{\sT,s}^a}.
\]
For a cutoff or linear kernel with finite scale, if the corresponding guess is correct and $\mu_n\geq\mu_{\star,s}^a$, every observation with positive weight has conditional mean
$\mu_i\geq(1-\theta)\mu_{\star,s}^a>0$.
Giving these observations greater weight can help the expert choose positive stakes during violations, without affecting false-alarm control.

\paragraph{Lipschitz Regularity.}
The regularity guesses guide which historical observations are useful for predicting the next score.
Fix an expert $a\in\aA_s$.
Suppose that its spatial regularity guess is valid and its spatial kernel is a cutoff or linear kernel with scale
\begin{align*}
    r_{\aX,s}^a
    =\frac{\theta\mu_{\star,s}^a}{\alpha_{\aX,s}^a}.
\end{align*}
A positive prediction weight requires the spatial kernel to be positive.
For a cutoff or linear kernel, this is possible only when the distance from the historical state to the current state is at most the spatial scale.
Hence,
\begin{align*}
    \omega_{s,n,i}^a>0
    \quad\Longrightarrow\quad
    \|\rX_{t_i}-\rX_{t_n}\|\leq r_{\aX,s}^a.
\end{align*}
If the spatial regularity guess is valid, the conditional means at these states therefore satisfy
\begin{align*}
    |\mu_i-\mu_n|
    &=|g_h(\rX_{t_i})-g_h(\rX_{t_n})|\\
    &\leq\alpha_{\aX,s}^a\|\rX_{t_i}-\rX_{t_n}\|
    \leq\alpha_{\aX,s}^a r_{\aX,s}^a
    =\theta\mu_{\star,s}^a.
\end{align*}
Thus, the chosen radius ensures that every observation with positive weight has a conditional mean within $\theta\mu_{\star,s}^a$ of the current conditional mean.
The same conclusion follows from a valid temporal guess and the corresponding temporal kernel.
In particular, if $\mu_n\geq\mu_{\star,s}^a$, then
\begin{align*}
    \omega_{s,n,i}^a>0
    \quad\Longrightarrow\quad
    \mu_i\geq(1-\theta)\mu_{\star,s}^a>0.
\end{align*}

Hence, when the current conditional mean is sufficiently positive, the selected observations also come from positive-mean steps.
Their realised scores may nevertheless be negative, and their second moments may differ.
Geometric kernels retain observations beyond the chosen scale, so they favour relevant observations without ensuring this implication.
These regularity guesses guide prediction and need not hold for false-alarm control.

\paragraph{Accumulating evidence.}
A positive expected log return favours wealth growth, but an alarm requires the realised wealth to reach the prescribed threshold.
The following lemma connects the evidence accumulated by one expert to the monitor's verdict.

\begin{lemma}[Sufficient evidence for detection]
Fix a starting index $s$, an expert $a\in\mathcal A_s$ with $c_s\eta_{s,a}>0$, and an index $n>s$.
Suppose that this expert uses stakes strictly below one.
If
\begin{align}
    \sum_{k=s}^{n-1}\log(1+\lambda_{s,k}^aW_k)
    \geq
    \log\frac{1}{\delta c_s\eta_{s,a}},
    \label{eq:sufficient-log-evidence}
\end{align}
then $\tau_\delta\leq n$.
If its stakes are at most $1/2$, it suffices that
\begin{align}
    \sum_{k=s}^{n-1}
    \left(
        \lambda_{s,k}^aW_k
        -(\lambda_{s,k}^a)^2W_k^2
    \right)
    \geq
    \log\frac{1}{\delta c_s\eta_{s,a}}.
    \label{eq:sufficient-quadratic-evidence}
\end{align}
\end{lemma}

\begin{proof}
The mixture launched at index $s$ receives $c_s$ units of capital, of which the fraction $\eta_{s,a}$ is assigned to expert $a$.
The account $G_{s,n}^a$ starts from one and records how this initial capital grows through betting.
Hence, the expert's contribution to the total wealth at index $n$ is $c_s\eta_{s,a}G_{s,n}^a$.

By definition, the total wealth is
\begin{align*}
    M_n
    =1-C_n
    +\sum_{r=0}^{n-1}c_r
      \sum_{b\in\mathcal A_r}\eta_{r,b}G_{r,n}^b.
\end{align*}
Since $n>s$, this sum includes the contribution of expert $a$ launched at $s$.
The reserve $1-C_n$ and all other contributions are nonnegative.
Keeping only the chosen expert's contribution therefore gives the lower bound
\begin{align*}
    M_n\geq c_s\eta_{s,a}G_{s,n}^a.
\end{align*}
Thus, it suffices to show that this expert's contribution alone reaches the alarm threshold $1/\delta$.
The expert's account is the product of its successive betting factors:
\begin{align*}
    G_{s,n}^a
    =\prod_{k=s}^{n-1}(1+\lambda_{s,k}^aW_k).
\end{align*}
Each factor is strictly positive because $W_k\geq-1$ and $\lambda_{s,k}^a<1$.
We can therefore express the product as the exponential of the sum of its logarithms.
Under the first condition of the lemma,
\begin{align*}
    G_{s,n}^a
    &=\exp\left(
        \sum_{k=s}^{n-1}\log(1+\lambda_{s,k}^aW_k)
    \right) \geq\exp\left(
        \log\frac{1}{\delta c_s\eta_{s,a}}
    \right)=\frac{1}{\delta c_s\eta_{s,a}}
\end{align*}
Multiplying by the expert's initial capital gives
\begin{align*}
    M_n
    \geq c_s\eta_{s,a}G_{s,n}^a
    \geq c_s\eta_{s,a}
         \frac{1}{\delta c_s\eta_{s,a}}
    =\frac{1}{\delta}.
\end{align*}
Since $\tau_\delta$ is the first index at which the total wealth reaches $1/\delta$, it follows that $\tau_\delta\leq n$.

Finally, suppose that all stakes of this expert are at most $1/2$.
Applying \eqref{eq:quadratic-log-lower-bound} to each betting factor and summing yields
\begin{align*}
    \sum_{k=s}^{n-1}\log(1+\lambda_{s,k}^aW_k)
    \geq
    \sum_{k=s}^{n-1}
    \left(
        \lambda_{s,k}^aW_k
        -(\lambda_{s,k}^a)^2W_k^2
    \right).
\end{align*}
If the second condition holds, the right-hand side is at least
$\log(1/(\delta c_s\eta_{s,a}))$.
The sum of log returns is therefore also at least this threshold, which establishes the first condition.
The preceding argument then gives $\tau_\delta\leq n$.
\end{proof}

The lemma explains the cost of dividing capital among experts.
An expert with less initial capital must accumulate more log return to trigger an alarm by itself.
For example, dividing a mixture equally among $K$ experts adds $\log K$ to this sufficient threshold compared with assigning that mixture to one expert.
This cost allows the monitor to support several strategies without selecting one in advance.
Since other accounts also contribute wealth, the monitor may alarm before any individual expert satisfies the condition.

\paragraph{Monitoring several conditions.}
A certificate may contain several expected-decrease conditions.
We can monitor them jointly by dividing the permitted false-alarm probability among their monitors. This union-bound approach allows the monitors to use the same observations, because independence is not required.

\begin{corollary}[Composition]
\label{cor:certificate-composition}
Suppose that a certificate has $J$ expected-decrease conditions.
For each condition $j$, run a monitor satisfying
Theorem~\ref{thm:monitor-validity} at a level $\delta_j>0$,
where $\sum_{j=1}^{J}\delta_j\leq\delta$.
Let $\tau$ be the first index at which any component monitor
or sampled domain-bound check raises an alarm.
If all certificate conditions hold for the deployed process,
then $\prob(\tau<\infty)\leq\delta$.
\end{corollary}

\begin{proof}
A valid domain bound cannot be violated at an observed state.
Writing $\tau_j$ for the first alarm index of component
monitor $j$, Theorem~\ref{thm:monitor-validity} and the union
bound give
\begin{align*}
    \prob(\tau<\infty)
    \leq\sum_{j=1}^{J}\prob(\tau_j<\infty)
    \leq\sum_{j=1}^{J}\delta_j
    \leq\delta.
\end{align*}
\end{proof}

\paragraph{Monitor power.}
To quantify detection, we consider a block on which the conditional
expected corrected drift $E_k$ exceeds the expected slack in the observable drift bound.
The difference determines the signal available to the monitor.
The following result uses an additional upper score bound and
an expert whose stake is constant throughout the block.
For the next lemma, fix a starting index $s$, an expert
$a\in\aA_r$ launched at $r\leq s$, a duration
$N\in\NN_{\geq1}$, and $\beta\in(0,1)$.
Suppose that the expert's capital at the start of the block is
$\rho_{r,s}\coloneqq c_r\eta_{r,a}G_{r,s}^a>0$ almost surely.
Let $d_\star$ be a lower bound on the conditional expected
corrected drift $E_k$, and let $s_\star\geq0$ be an upper bound
on the conditional expected slack
$\expe[\overline L_k-L_k\mid\aF_k]$ throughout the block.
We require $d_\star>s_\star$, so a positive excess remains
after accounting for the slack.
Let $B_{\max}>0$ be an upper bound on $B_k$, and assume
$\varepsilon\coloneqq(d_\star-s_\star)/B_{\max}\in(0,1]$.
Formally, for every $k=s,\ldots,s+N-1$, we require almost surely
\begin{align*}
    E_k&\geq d_\star,
    &\expe[\overline L_k-L_k\mid\aF_k]&\leq s_\star,\\
    B_k&\leq B_{\max},
    &W_k&\leq1,
    &\lambda_{r,k}^a&=\varepsilon/2.
\end{align*}
The standing score assumptions already give $W_k\geq-1$,
so the scores lie in $[-1,1]$.

\begin{lemma}[Detection of a persistent certificate violation]
\label{lem:certificate-detection}
On the $\aF_s$-measurable event
\begin{align}
    N\geq
    \frac{8B_{\max}^2}{(d_\star-s_\star)^2}
    \log\frac{1}{\delta\rho_{r,s}\beta},
    \label{eq:persistent-mean-duration}
\end{align}
the monitor satisfies
$\prob(\tau_\delta\leq s+N\mid\aF_s)\geq1-\beta$
almost surely.
\end{lemma}

\begin{proof}[Proof of Lemma~\ref{lem:certificate-detection}]
First, we translate the certificate excess into a lower bound on the conditional score mean. Since $B_k$ is positive and $\aF_k$-measurable,
\begin{align*}
    \mu_k
    &=\frac{E_k-\expe[\overline L_k-L_k\mid\aF_k]}{B_k} \geq\frac{d_\star-s_\star}{B_k}
    \geq\frac{d_\star-s_\star}{B_{\max}}
    =\varepsilon.
\end{align*}
Thus, each score in the block has conditional mean at least $\varepsilon$.

Write $Z_k=W_k-\mu_k$. Each $Z_k$ has conditional mean zero and conditional range of length at most $2$, because $W_k\in[-1,1]$. The conditional Hoeffding bound gives, for every $q\geq0$,
\begin{align*}
    \expe[e^{-qZ_k}\mid\aF_k]\leq e^{q^2/2}.
\end{align*}
For $j=0,\ldots,N$, write
$S_j\coloneqq\sum_{k=s}^{s+j-1}Z_k$, with $S_0=0$.
Since $S_j$ is $\aF_{s+j}$-measurable, the tower property
and the conditional Hoeffding bound give
\begin{align*}
    \expe[e^{-qS_{j+1}}\mid\aF_s]
    &=\expe\!\left[
        e^{-qS_j}
        \expe[e^{-qZ_{s+j}}\mid\aF_{s+j}]
        \,\middle|\,\aF_s
    \right]\\
    &\leq e^{q^2/2}\expe[e^{-qS_j}\mid\aF_s].
\end{align*}
Applying this inequality successively for $j=0,\ldots,N-1$,
starting from $e^{-qS_0}=1$, yields
\begin{align*}
    \expe[e^{-qS_N}\mid\aF_s]\leq e^{Nq^2/2}.
\end{align*}

For $q>0$, the event $S_N\leq-u$ is equivalent to
$e^{-qS_N}\geq e^{qu}$.
Conditional Markov's inequality therefore gives
\begin{align*}
    \prob(S_N\leq-u\mid\aF_s)
    &=\prob(e^{-qS_N}\geq e^{qu}\mid\aF_s)\\
    &\leq e^{-qu}\expe[e^{-qS_N}\mid\aF_s]\\
    &\leq\exp(-qu+Nq^2/2).
\end{align*}
For $u>0$, the exponent is minimised at $q=u/N$.
Substituting this choice gives
\begin{align*}
    \prob(S_N\leq-u\mid\aF_s)
    \leq\exp\left(-\frac{u^2}{2N}\right).
\end{align*}
For $u=0$, the same bound holds trivially.

Now choose $u=\sqrt{2N\log(1/\beta)}$, so the right-hand
side equals $\beta$.
Thus, with conditional probability at least $1-\beta$,
$S_N\geq-\sqrt{2N\log(1/\beta)}$.
Since $W_k=\mu_k+Z_k$ and $\mu_k\geq\varepsilon$
throughout the block, on this event we have
\begin{align*}
    \sum_{k=s}^{s+N-1}W_k
    &=\sum_{k=s}^{s+N-1}\mu_k+S_N\geq N\varepsilon-\sqrt{2N\log(1/\beta)}.
\end{align*}

We next translate the score bound into growth of the expert's
wealth.
Since $\rho_{r,s}>0$, the account wealth $G_{r,s}^a$ is positive.
During the block, the expert uses the constant stake
$\lambda=\varepsilon/2\leq1/2$.
Each betting factor satisfies $1+\lambda W_k\geq1/2$, so
\begin{align*}
    \frac{G_{r,s+N}^a}{G_{r,s}^a}
    =\prod_{k=s}^{s+N-1}(1+\lambda W_k)>0.
\end{align*}
Taking logarithms converts this product into a sum.
For each score, the inequality
$\log(1+\lambda w)\geq\lambda w-\lambda^2w^2$
and the bound $W_k^2\leq1$ give
\begin{align*}
    \log(1+\lambda W_k)
    \geq\lambda W_k-\lambda^2W_k^2
    \geq\lambda W_k-\lambda^2.
\end{align*}
Summing over the $N$ scores therefore yields
\begin{align*}
    \log\frac{G_{r,s+N}^a}{G_{r,s}^a}
    &\geq\lambda\sum_{k=s}^{s+N-1}W_k-N\lambda^2.
\end{align*}
On the score event established above, substituting
$\lambda=\varepsilon/2$ gives
\begin{align*}
    \log\frac{G_{r,s+N}^a}{G_{r,s}^a}
    &\geq
    \frac{\varepsilon}{2}
    \left(N\varepsilon-\sqrt{2N\log(1/\beta)}\right)
    -\frac{N\varepsilon^2}{4}=\frac{N\varepsilon^2}{4}
      -\varepsilon\sqrt{\frac N2\log(1/\beta)}.
\end{align*}

To simplify the square-root term, apply
$2\sqrt{uv}\leq u+v$ with
$u=N\varepsilon^2/8$ and $v=\log(1/\beta)$:
\begin{align*}
    \varepsilon\sqrt{\frac N2\log(1/\beta)}
    &=2\sqrt{\frac{N\varepsilon^2}{8}\log(1/\beta)}
    \leq\frac{N\varepsilon^2}{8}+\log(1/\beta).
\end{align*}
Consequently,
\begin{align*}
    \log\frac{G_{r,s+N}^a}{G_{r,s}^a}
    \geq\frac{N\varepsilon^2}{8}-\log(1/\beta).
\end{align*}

It remains to show that this growth suffices to reach the
alarm threshold.
Since $\varepsilon=(d_\star-s_\star)/B_{\max}$,
the duration condition~\eqref{eq:persistent-mean-duration}
implies
\begin{align*}
    \frac{N\varepsilon^2}{8}\geq\log\frac{1}{\delta\rho_{r,s}\beta}=\log\frac{1}{\delta\rho_{r,s}}+\log(1/\beta).
\end{align*}
Combining the two bounds and exponentiating gives
\begin{align*}
    \log\frac{G_{r,s+N}^a}{G_{r,s}^a}
    \geq\log\frac{1}{\delta\rho_{r,s}}
    \quad\Longrightarrow\quad
    \frac{G_{r,s+N}^a}{G_{r,s}^a}
    \geq\frac{1}{\delta\rho_{r,s}}.
\end{align*}

The expert starts the block with capital
$\rho_{r,s}=c_r\eta_{r,a}G_{r,s}^a$.
Multiplying this capital by its growth factor gives its
contribution to the total wealth at $s+N$.
Since the reserve and all other accounts are nonnegative,
\begin{align*}
    M_{s+N}
    \geq c_r\eta_{r,a}G_{r,s+N}^a
    =\rho_{r,s}\frac{G_{r,s+N}^a}{G_{r,s}^a}
\geq\rho_{r,s}\frac{1}{\delta\rho_{r,s}}
    =\frac1\delta.
\end{align*}
Thus, the monitor has raised an alarm by index $s+N$.

The score event has conditional probability at least $1-\beta$
given $\aF_s$, and the duration condition is $\aF_s$-measurable.
Hence, on that condition,
\begin{align*}
    \prob(\tau_\delta\leq s+N\mid\aF_s)\geq1-\beta
    \qquad\text{almost surely}.
\end{align*}
\end{proof}
The sufficient duration increases quadratically as the observable excess $d_\star-s_\star$ decreases, and logarithmically as the expert's capital at the start of the block decreases. For an expert launched at $s$, this capital is $\rho_{s,s}=c_s\eta_{s,a}$. For an expert launched earlier, $\rho_{r,s}$ also accounts for its previous gains and losses. A conservative drift bound can conceal a violation if its expected slack is as large as the excess. With observation gap $h$, the $N$ scores span physical time $hN$. The upper score bound is used only for this detection result; false-alarm control still requires only $W_k\geq-1$.

\section{Implementation}
\label{sec:monitor-implementation}

Algorithm~\ref{alg:fixed-expert-monitor} implements the monitor using a shared history of the most recent $m$ completed scores and their starting states. Before each new observation, it chooses the funding amount, computes the loss bound, and uses the retained history to choose the experts' stakes. Upon observing the next state, it computes the score, updates the accounts, and checks whether the total wealth reaches $1/\delta$. The completed score is then added to the history.

\paragraph{Combining accounts.}
Repeated funding need not require storing a separate account for every starting index. Suppose that every mixture uses the same finite set of experts $\aA_s=\aA$, with $|\aA|=K$, and the same nonnegative weights $\eta_{s,a}=\eta_a$ satisfying $\sum_{a\in\aA}\eta_a=1$. Each expert keeps its kernels, scales, and stake strategy fixed and uses the shared retained history. For the moment and log-return strategies above, its stake therefore does not depend on the starting index:
\begin{align*}
    \lambda_{s,n}^a=\lambda_n^a
    \qquad\text{for every }s\leq n.
\end{align*}
The accounts of expert $a$ may hold different amounts of capital, but the next score multiplies each of them by the same factor $1+\lambda_n^aW_n$. We can therefore store their combined capital:
\begin{align*}
    A_n^a\coloneqq\sum_{s=0}^{n-1}c_s\eta_aG_{s,n}^a,
    \qquad R_n\coloneqq1-C_n,
\end{align*}
where $R_n$ is the remaining reserve. Funding at index $n$ adds $c_n\eta_a$ to the combined account of expert $a$ and removes $c_n$ from the reserve. Applying the common betting factor gives
\begin{align}
    A_{n+1}^a&=(A_n^a+c_n\eta_a)(1+\lambda_n^aW_n),
    &R_{n+1}&=R_n-c_n,
    \label{eq:aggregated-accounts}
\end{align}
with $A_0^a=0$, $R_0=1$, and total wealth $M_n=R_n+\sum_{a\in\aA}A_n^a$.
Multiplication distributes over the individual account balances, so this update gives exactly the same total wealth and alarm time as maintaining those accounts separately. The optimisation requires the combined accounts to use identical stakes. Accounts whose scales or prediction histories depend on their starting indices must generally remain separate.

\begin{algorithm}[t]
\caption{Monitor with combined expert accounts}
\label{alg:fixed-expert-monitor}
\begin{algorithmic}[1]
\Require Certificate $f$, allowed drift $l$, spacing $h$, level $\delta$, history size $m$;
\Statex \hspace{\algorithmicindent} valid score-bound routines; a fixed set $\aA$ of $K$ experts and weights $\eta_a$;
\Statex \hspace{\algorithmicindent} predictable funding rule $0\leq c_n\leq R_n$.
\State Observe $\rX_{t_0}$; initialise $R_0\gets1$, $A_0^a\gets0$ for every $a\in\aA$, and empty history.
\For{$n=0,1,\ldots$ while observations continue}
    \State Choose $c_n$ and compute $B_n$ using information available at $t_n$.
    \State Compute prediction weights from the retained history using~\eqref{eq:prediction-weights}.
    \State If $\rX_{t_n}\in \aS$ choose each expert's stake $\lambda_n^a$ using~\eqref{eq:moment-stake} , or~\eqref{eq:log-return-stake}, otherwise $\lambda_n^a=0$.
    \State Observe $\rX_{t_{n+1}}$ and compute $\overline L_n$ and $W_n$.
    \State Update the combined accounts and reserve using~\eqref{eq:aggregated-accounts}.
    \If{$R_{n+1}+\sum_{a\in\aA}A_{n+1}^a\geq1/\delta$}
        \State \Return alarm at index $n+1$.
    \EndIf
    \State Append $(t_n,\rX_{t_n},W_n)$ to the history and retain its most recent $m$ entries.
\EndFor
\end{algorithmic}
\end{algorithm}

The history stores each score with its starting state because spatial prediction compares the states from which the corresponding transitions began. Removing an old observation from the history affects future stake choices but leaves accumulated wealth unchanged. If observations end upon reaching a target or another observed termination condition, the monitor processes the final completed interval before stopping.

\paragraph{Computational cost.}
For states in $\RN^d$, the retained history requires $O(md)$ storage and the combined accounts require $O(K)$. Processing prediction weights one expert at a time therefore gives total storage $O(md+K)$ for the betting implementation.
At each observation, distances to the retained states can be computed once and shared across experts, costing $O(md)$. Each moment strategy then computes two weighted sums in $O(m)$ time, giving $O(md+Km)$ work per observation. For log-return stakes, each derivative evaluation scans the retained history. Using $J$ bisection iterations therefore gives $O(md+KmJ)$ work. Finite bisection approximates the maximiser, but any returned stake in $[0,\overline\lambda]$ preserves false-alarm control. These bounds exclude the resources required by the certificate, score-bound, and funding routines.

\paragraph{Recursive updates for unlimited history.}
For uniform or purely temporal geometric weights, an expert using the moment strategy can use the entire history without storing individual observations.
Since its temporal scale is fixed across launches, write
$r_{\sT}^a\coloneqq r_{\sT,s}^a$, independently of $s$.
For the geometric kernel $k(u)=2^{-u}$, set
$\gamma^a=2^{-h/r_{\sT}^a}$; for uniform weighting, set $\gamma^a=1$.
The geometric weights are
$2^{-(t_n-t_i)/r_{\sT}^a}=(\gamma^a)^{n-i}$.
Rescaling them to $(\gamma^a)^{n-1-i}$ leaves the moment stake unchanged, since the common factor cancels in the ratio of weighted sums.
Starting from $S_0^a=U_0^a=H_0^a=0$, we maintain the resulting sums through
\begin{align*}
    S_{n+1}^a&=\gamma^a S_n^a+W_n,
    &
    U_{n+1}^a&=\gamma^a U_n^a+W_n^2,
    &
    H_{n+1}^a&=\gamma^a H_n^a+1.
\end{align*}
For $n\geq1$, the weighted mean and second moment are $S_n^a/H_n^a$ and $U_n^a/H_n^a$. The moment stake uses $S_n^a$ and $U_n^a$ directly, so maintaining $H_n^a$ is optional. When $\gamma^a=1$, these recursions give the ordinary sums and $H_n^a=n$; when $\gamma^a<1$, each update discounts earlier observations. For $K$ such experts, maintaining these sums requires $O(K)$ work per observation and $O(K)$ storage, independently of the history length. This optimisation does not generally extend to spatial weights, which change with the current state, or to exact log-return optimisation, whose objective cannot be recovered from these sums alone.

\section{Experimental setup}
\label{sec:selected-experimental-setup}

\paragraph{Prediction, stakes and funding.}
Uniform prediction (U) assigns equal weight to all retained scores.
Temporal prediction (T) favours recent scores, while spatial prediction
(S) favours scores whose starting states are close to the current state.
Temporal and spatial prediction use the linear kernel
$k(u)=(1-u)_+$ with fixed scales.

For the synthetic benchmarks, we retain $16384$ observations.
We use five scales, labelled XS--XL, calibrated to reference effective
sample sizes $32,128,512,2048,8192$, where
\[
    \operatorname{ESS}(w)
    =\frac{(\sum_i w_i)^2}{\sum_i w_i^2}.
\]
The temporal cutoffs are approximately $43.166$, $171.167$, $683.167$,
$2731.167$ and $10923.167$ observations.
We evaluate the individual scales, an equal mixture, a centred mixture
with weights $(0.1,0.2,0.4,0.2,0.1)$, and five equal mixtures that each
omit one scale. 
For the control and neural benchmarks, we use U-L, T-L and S-L with
a history of $8192$ observations.
T-L equally mixes temporal cutoffs $64,128,512,2048,8192$, while S-L
equally mixes spatial radii $0.00625,0.025,0.1,0.4,1.6$. 
Log-return optimisation uses endpoint checks followed, when necessary,
by $32$ bisection steps.
We launch accounts every $256$ observations.
Launch $j\geq0$ receives $6/[\pi^2(j+1)^2]$ of the initial capital
in the synthetic benchmarks and $1/[(j+1)(j+2)]$ in the control and
neural benchmarks.

\subsection{Synthetic benchmarks}
\label{sec:selected-synthetic-setup}

The synthetic benchmarks isolate different sources of detection
difficulty
(Figures~\ref{fig:synthetic:overview-rate}--\ref{fig:synthetic:l-m-delay}).
We consider six families of processes.

\emph{Temporal:}
The scalar process starts at $\rX_0=32769$ and follows
\[
    \rX_{n+1}=\rX_n-0.001+g_n+\xi_n,
    \qquad
    \xi_n\overset{\mathrm{i.i.d.}}\sim
    \operatorname{Unif}[-0.6,0.6].
\]
With certificate $f(x)=x$ and allowance $-0.001$, the conditional
expected corrected increment is $\expe[D_n\mid\aF_n]=g_n$.
Persistent violations use $g_n\in\{0.01,0.02,0.05\}$ throughout.
Delayed violations use $g_n=-0.02$ until onset
$2048$, $4096$ or $8192$, followed by $g_n=0.05$.
Temporary violations use $g_n=0.05$ for $256$, $1024$ or $4096$
observations from onset $8192$, and $g_n=-0.02$ otherwise.
Each run lasts $32768$ observations.

\emph{Spatial:}
The first cyclic system follows a deterministic circular path with
directly prescribed scores. It does not use a controller--certificate
pair, and its conditional score mean jumps at the sector boundaries,
so it is not globally Lipschitz.
Each period starts with $64$ observations of mean $0.10$ in a fixed
sector, followed by observations of mean $-0.02$.
Independent symmetric score noise has amplitude $0.3$.
We use periods $P\in\{512,2048,4096\}$ and horizons $8P$.
Each run therefore contains exactly $512$ positive-mean observations;
increasing $P$ lengthens the gaps between visits.

The second system is a stochastic ring with noise in both radial
position and angular speed. The certificate is the squared radial
error, and the allowance is its exact nominal expected increment.
The radial contraction is $0.2$, radial noise is uniform on
$[-0.01,0.01]$, and angular speed has independent uniform modulation
of $\pm20\%$.
A triangular angular envelope with half-width $0.2$ reduces the radial
feedback by up to $75\%$.
We use the same periods and horizons, with a nominal sector visit of
$256$ observations. Stochastic return times make the number of
observations under disturbance vary across runs.

\emph{Hill:}
We use the robot from the running example.
It starts at $\rX_0=10$ and stops upon reaching $[0,1]$.
Before stopping,
\[
    \rX_{n+1}=\rX_n-0.01+0.02Z_n,
    \qquad
    Z_n\mid\aG_n\sim
    \operatorname{Bernoulli}
    \bigl(1/4+a(1-|\rX_n-5|/2)_+\bigr).
\]
With certificate $f(x)=x$, allowance $-0.005$ and loss bound $0.02$,
the conditional expected score is $a(1-|\rX_n-5|/2)_+$.
Amplitudes $a\in\{0.0625,0.125,0.25\}$ reduce expected progress at
the hill centre by one quarter, one half or entirely.
Stronger disturbances also keep the robot on the hill for longer.

For the cyclic systems, increasing the period changes the horizon,
the observations available in the retained history and the number of
funding opportunities, even though positive exposure is fixed for
direct recurrence.
Including the scalar, ring and hill references gives $21$ settings.
Direct recurrence has no separate saved reference.
The six example trajectories use seed $970000003$ and settings with
weak violations, late onset, short duration or long gaps between visits.
These figures illustrate individual runs; the heatmaps summarise
all seeds.

\subsection{Control benchmarks}
\label{sec:selected-control-setup}

Building and Two Tanks adapt benchmarks from
PRoTECT~\citep{wooding2026protect}.
The stochastic double-integrator LQR uses a standard quadratic value
function~\citep{tedrake2023underactuated}.
Each system has one reference setting and two disturbances, each
evaluated at three strengths, giving $21$ settings in total.
All disturbances act from the start of the run.
The certificates and allowances remain fixed.

\emph{Building:}
We reduce heating effectiveness by $8\%$, $10\%$ or $12\%$, or make
the controller underestimate the temperature by $1$, $2$ or $3$
degrees. In the latter case, the monitor observes the true temperature.

\emph{Tanks:}
We add inlet flow of amplitude $0.2$, $0.3$ or $0.4$, which fades as
the downstream level rises from $2.15$ to $2.25$.
Alternatively, we add independent inlet bursts of size $0.5$ with
probability $0.35$, $0.40$ or $0.45$ per step.

Building and Tanks use bounded triangular noise with the same means
and variances as their source models.
Their newly synthesised sum-of-squares certificates have degrees two
and four, respectively.
Independent checks using exact arithmetic validate the nominal
allowances $0.04$ and $0.07$ under these noise distributions.
Initial states are uniform on $[19.5,20]$ for Building and on
$[4.45,4.55]\times[1.45,1.55]$ for Tanks.

\emph{LQR:}
The initial state lies in $[1,2]\times[-0.1,0.1]$.
The nominal acceleration diffusion is $0.1$.
We sample every $0.1$ time units using the exact Gaussian transition
and use the exact nominal sampled certificate increment as the allowance.
Disturbances either reduce feedback by $5\%$, $10\%$ or $20\%$,
or multiply the noise amplitude by $1.05$, $1.10$ or $1.20$.

All control runs have a limit of $32768$ observations.
Building and Tanks also stop upon entering an unsafe set or leaving
the domain; LQR stops upon leaving $[-10,10]^2$.
We compute conditional moments exactly along each recorded trajectory.
These determine the fraction of runs encountering a violation and the
expected-score summaries
(Figures~\ref{fig:control:overview-rate}--\ref{fig:control:expected_score}).

\subsection{Neural controller--certificate pairs}
\label{sec:selected-neural-setup}

We retain $16$ original controller--certificate pairs from
LogRASM: four CollisionAvoidance pairs,
six Linear layout-$0$ pairs, two Linear layout-$1$ pairs, two Pendulum
pairs and two Triple-integrator pairs.
The controller and certificate parameters remain fixed.
We consider three disturbance magnitudes per setting.
Initial states are sampled from fixed pools away from the target,
unsafe sets and domain boundaries, with at least $20\%$ headroom below
the certificate threshold. Starts are shared across disturbance
strengths and pairs within the same layout.

To obtain longer trajectories, we multiply the source drift increments
by $1/20$ and the direct plant-noise amplitudes by $1/\sqrt{20}$.
The baseline controller effectiveness is $0.1$ for CollisionAvoidance,
$0.5$ and $0.25$ for the two Linear layouts, $1$ and $0.5$ for the
ordinary and logarithmic Pendulum pairs, and $1$ for Triple integrator.

For CollisionAvoidance, we add terrain-dependent slips with peak
probabilities $0.25$, $0.35$ and $0.50$.
For the Linear systems, we add fluctuating sensor offsets with
amplitudes $0.30$, $0.35$ and $0.40$.
For Pendulum, random torque pushes the system away from the upright
position, with amplitudes $0.075$, $0.10$ and $0.125$.
For Triple integrator, we add a positive actuator bias with amplitudes
$0.06$, $0.08$ and $0.10$.
Paths stop upon reaching the target, entering an unsafe set, leaving
the domain, reaching the certificate threshold or completing $20480$
observations. 

Let $C=1/(1-\rho)$.
For an ordinary certificate output $V$, we use
$f(x)=\min\{V(x),C\}$.
For a logarithmic output $L$, we use
$f(x)=\exp(\min\{L(x),\log C\})$.
We evaluate each recorded path under three decrease requirements,
$\epsilon\in\{0,0.0005,0.001\}$.
The allowance is $-\epsilon$ for ordinary certificates and
$(e^{-\epsilon}-1)f(x)$ for exponentiated logarithmic certificates.
Thus, $\epsilon=0$ tests expected nonincrease; positive values require
additive or multiplicative decrease, respectively.
All three conditions use the same predictable loss bound and the same
physical trajectories
(Figures~\ref{fig:neural:overview-rate}
and~\ref{fig:neural:overview-delay}).

\section{Synthetic}
\label{sec:selected-synthetic-figures}

\begin{figure}[!htbp]
    \centering
    \includegraphics[width=\linewidth]{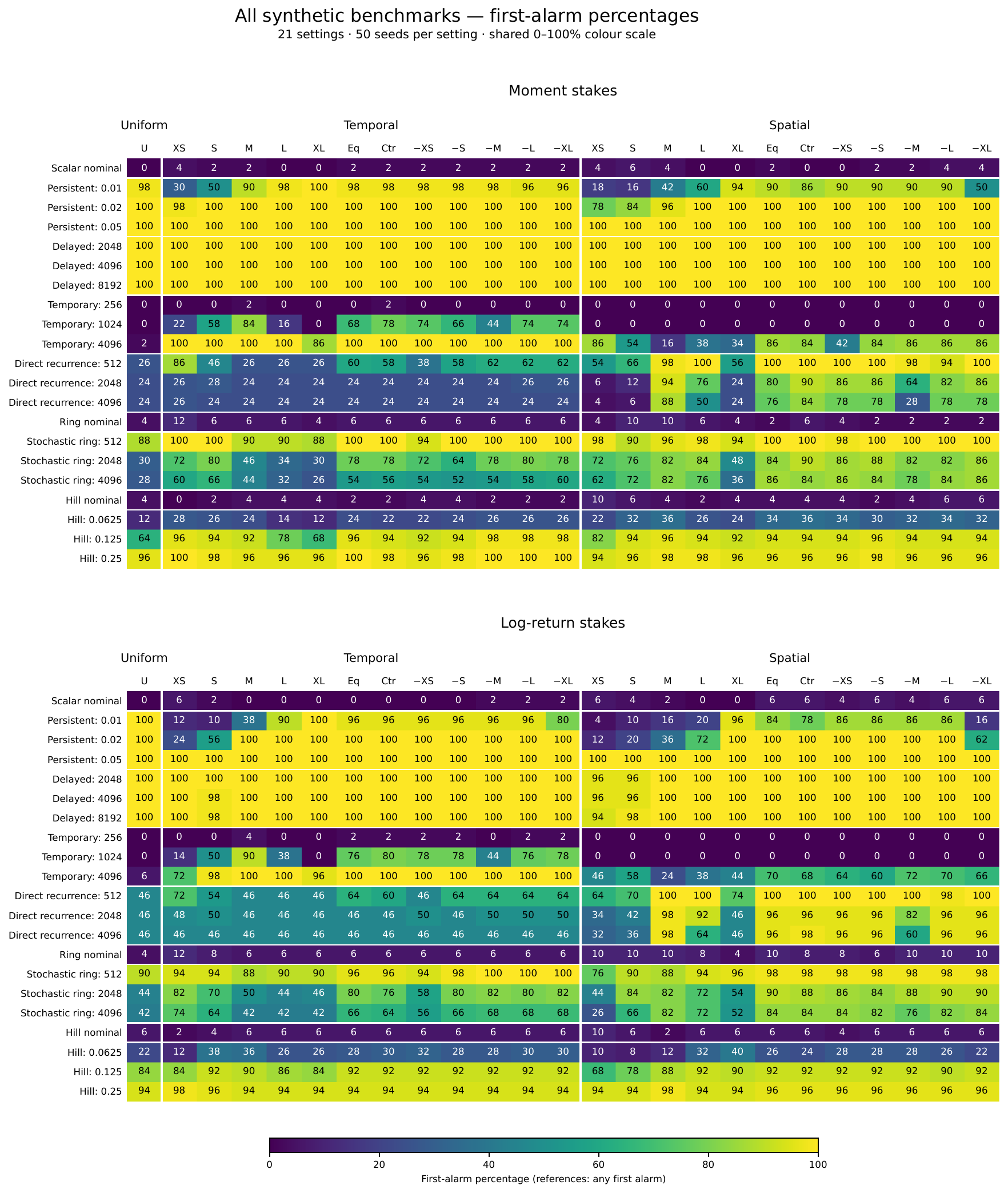}
    \caption{Detection rates for the synthetic benchmarks, with $21$ settings,
    $50$ monitors and $50$ paired runs per setting.
    Cells show the percentage of runs with an eligible first alarm;
    reference rows count any first alarm.
    The panels separate moment (M) and log-return (L) stakes.
    U uses uniform prediction, while T and S use temporal and spatial
    prediction. XS--XL denote individual scales; Eq is their equal
    mixture; Ctr uses weights $(0.1,0.2,0.4,0.2,0.1)$.
    A minus label indicates that the corresponding scale is omitted
    from an otherwise equal mixture.
    All monitors use wealth threshold $20$.
    The panels share a colour scale from $0$ to $100\%$.
    NA denotes incomplete data.}
    \label{fig:synthetic:overview-rate}
\end{figure}

\begin{figure}[!htbp]
    \centering
    \includegraphics[width=\linewidth]{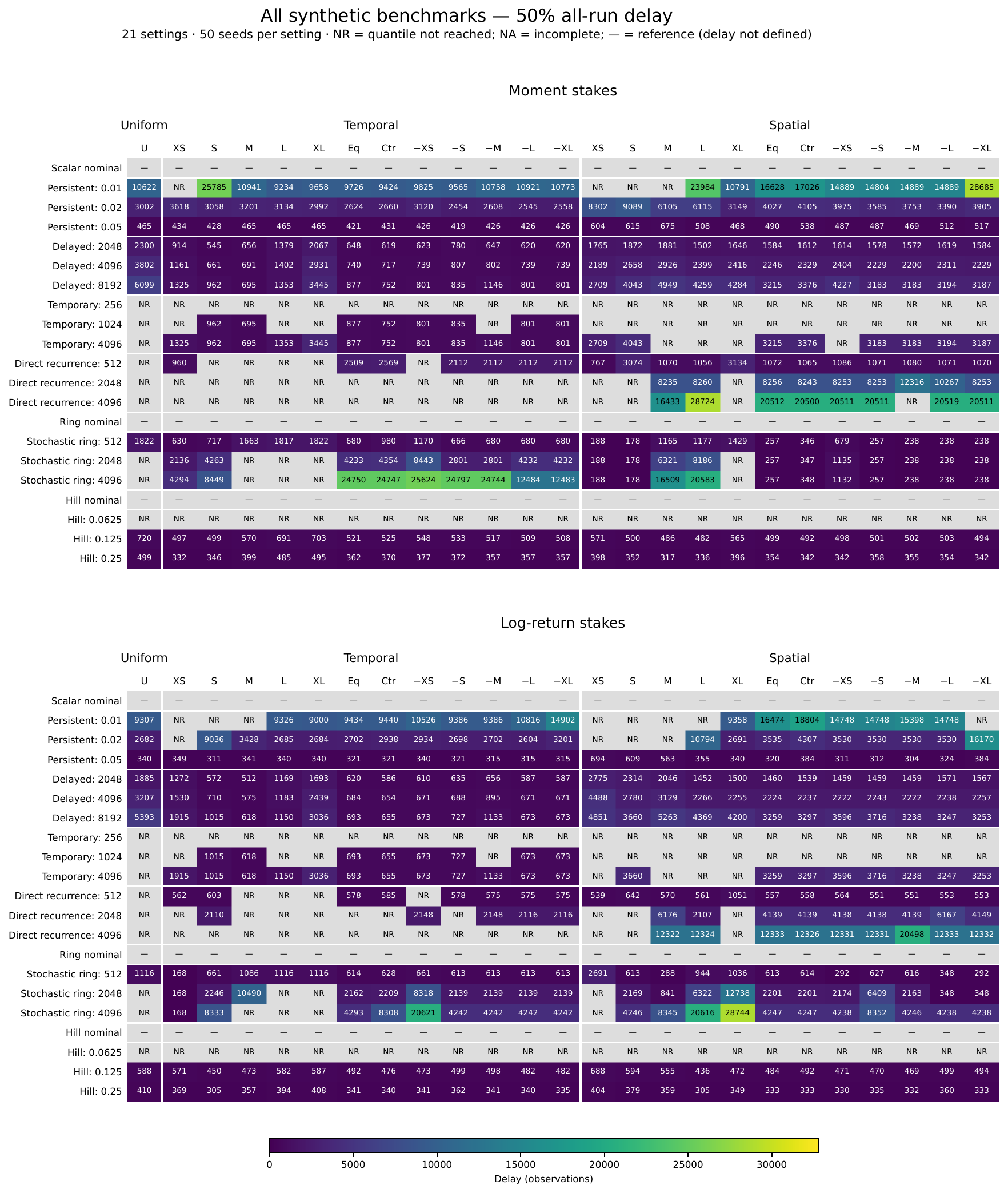}
    \caption{The $50\%$ detection delays for the synthetic monitors in
    Figure~\ref{fig:synthetic:overview-rate}.
    Each value is the first delay by which $25$ of all $50$ runs have
    eligible first alarms, measured in observations since the declared
    onset or first entry into the region with positive conditional mean.
    Nondetections, short paths and early first alarms remain in the
    denominator. A later crossing does not replace an early first alarm.
    NR denotes a median that was not reached, NA incomplete data,
    and a dash an undefined reference delay.
    Both panels use the same colour scale.}
    \label{fig:synthetic:overview-delay}
\end{figure}

\begin{figure}[!htbp]
    \centering
    \includegraphics[width=\linewidth]{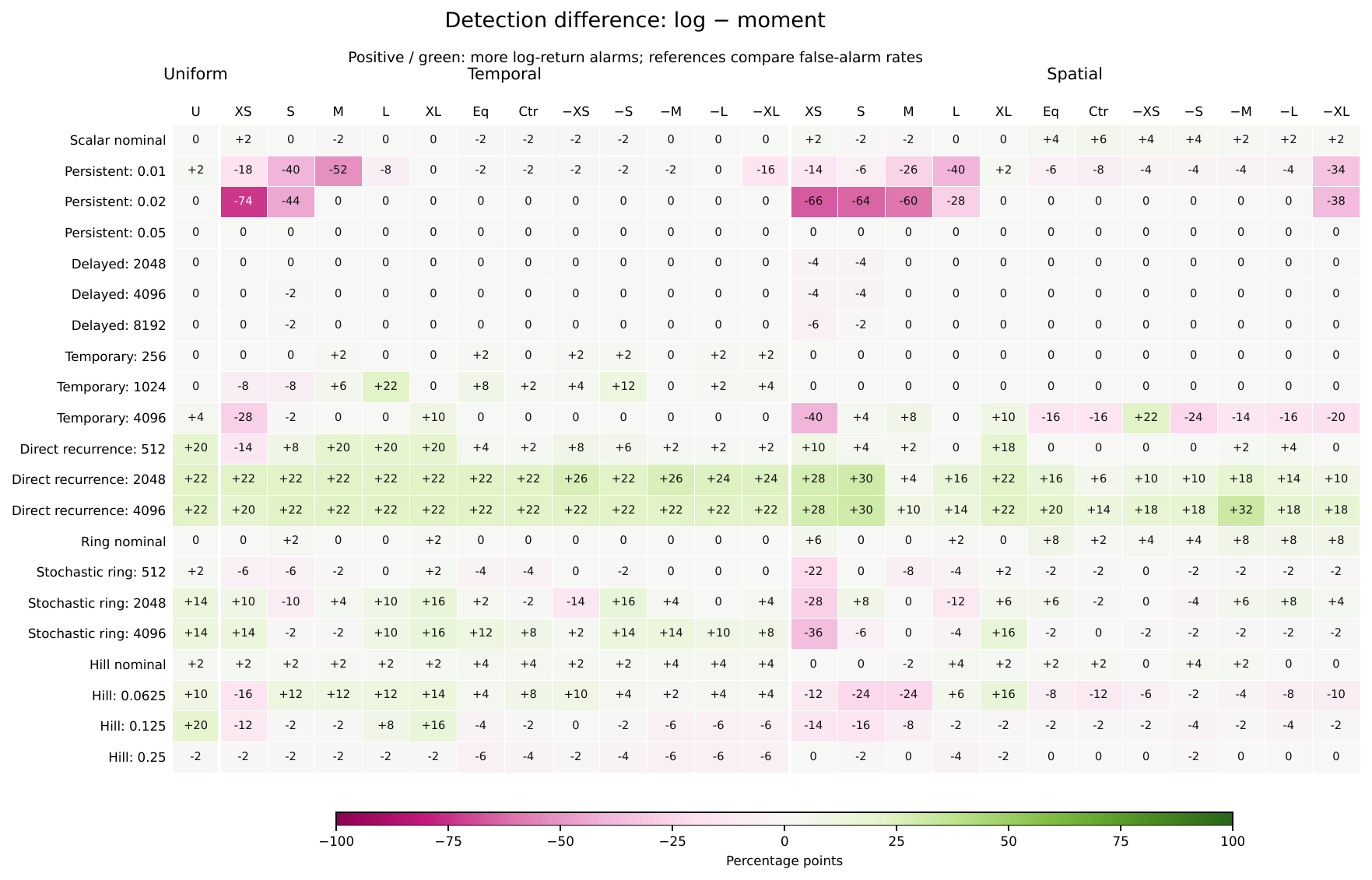}
    \caption{Differences in synthetic detection rates between log-return and
    moment stakes, computed as L$-$M in percentage points.
    Each comparison uses the same prediction rule and the same $50$ paths.
    Positive green cells indicate more alarms with log-return stakes;
    negative cells indicate more alarms with moment stakes.
    For reference rows, more alarms mean a higher false-alarm rate.
    Column labels follow Figure~\ref{fig:synthetic:overview-rate}.
    The comparison includes the different stake caps, $0.99$ for L
    and $0.5$ for M. These differences are descriptive; no significance
    test is applied.}
    \label{fig:synthetic:l-m-rate}
\end{figure}

\begin{figure}[!htbp]
    \centering
    \includegraphics[width=\linewidth]{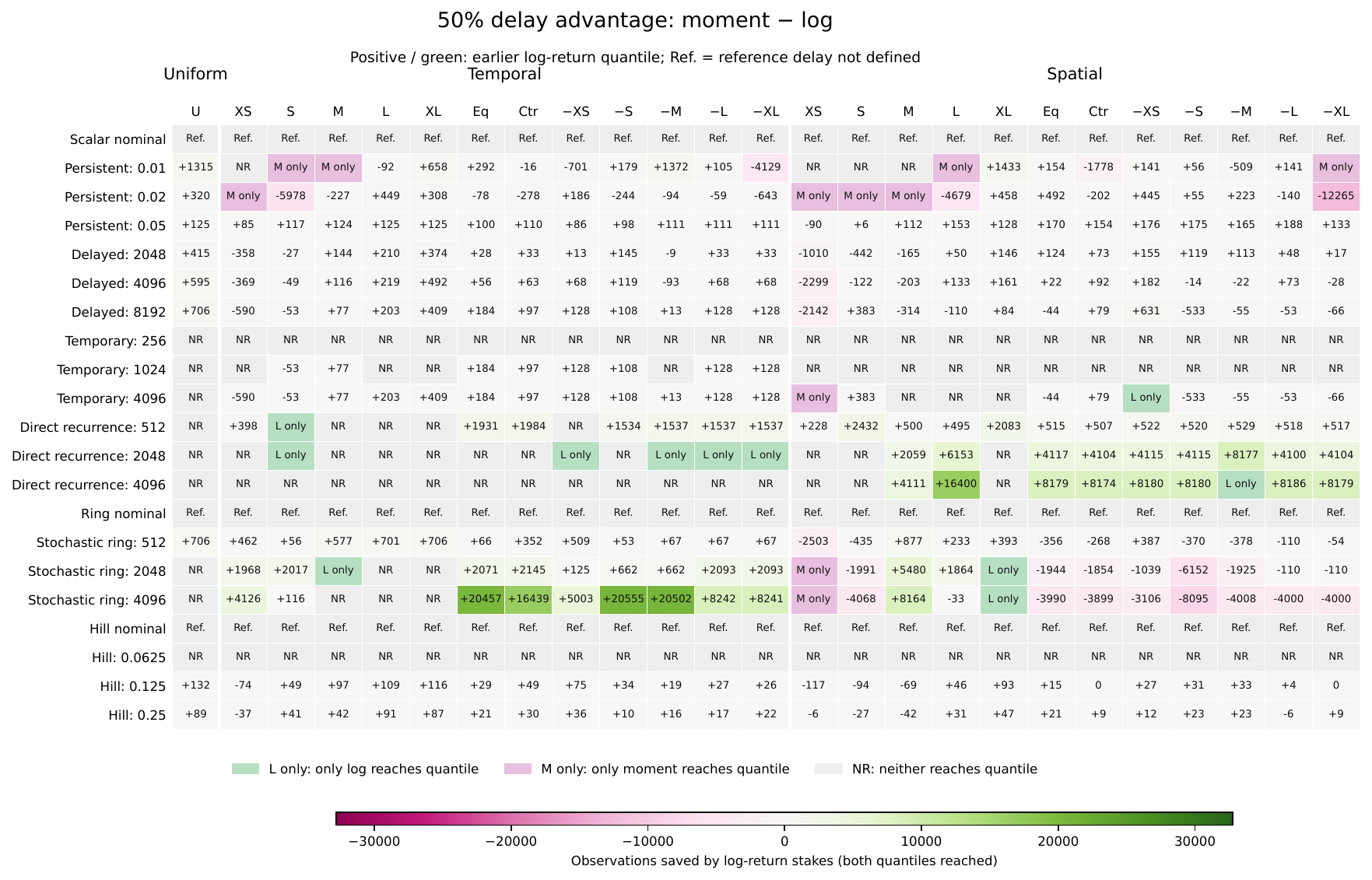}
    \caption{Differences in synthetic median detection delays, computed as
    $Q_{0.5}^{\mathrm M}-Q_{0.5}^{\mathrm L}$ in observations.
    Positive green values indicate earlier detection with log-return
    stakes when both strategies reach $25$ eligible first alarms among
    all $50$ runs.
    \textnormal{L only} and \textnormal{M only} indicate that only the
    named strategy reaches this fraction; NR means neither does.
    Reference delays are undefined (Ref.), and NA denotes incomplete data.
    No artificial delays are assigned to non-numeric cells.
    The comparisons use the same prediction rules and paths, with
    the respective stake caps of each strategy.}
    \label{fig:synthetic:l-m-delay}
\end{figure}

\begin{figure}[!htbp]
    \centering
    \includegraphics[width=\linewidth]{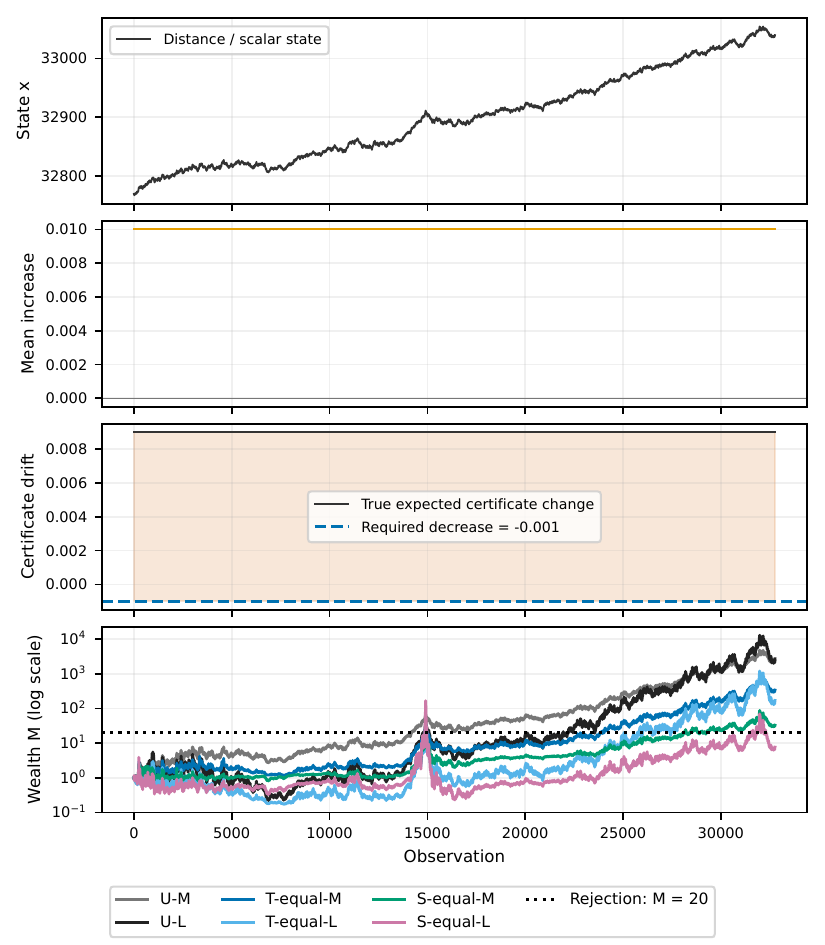}
    \caption{Persistent violation with $g=0.01$ and seed $970000003$.
    From top to bottom: the scalar state; mean increase $g$;
    the exact expected certificate increment $g-0.001=0.009$
    and allowance $-0.001$; and monitor wealth on a logarithmic scale.
    Uniform noise on $[-0.6,0.6]$ makes individual increments noisy
    despite the positive conditional mean of the corrected increment.
    U uses uniform prediction, while T-equal and S-equal mix the five
    XS--XL experts equally. M and L denote moment and log-return stakes.
    The dotted line marks wealth threshold $20$.
    All $32768$ observations are shown, including those after alarms.}
    \label{fig:synthetic:ex-persistent}
\end{figure}

\begin{figure}[!htbp]
    \centering
    \includegraphics[width=\linewidth]{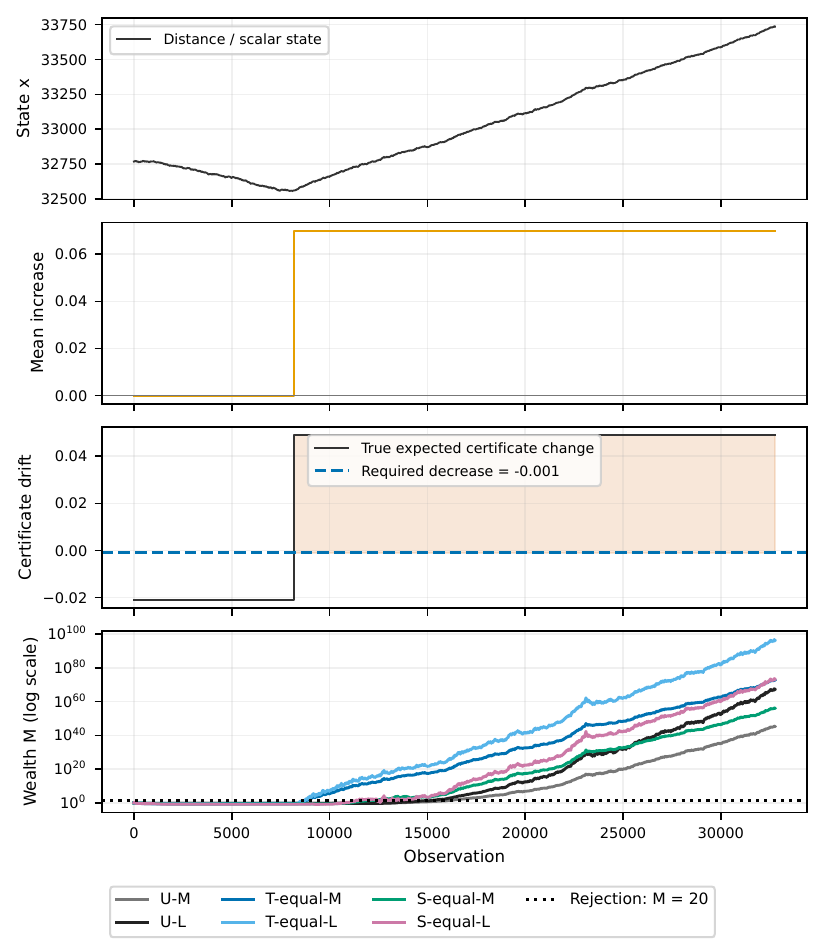}
    \caption{Delayed violation with onset $8192$ and seed $970000003$.
    The conditional mean of the corrected increment changes from
    $-0.02$ to $0.05$ and remains positive.
    From top to bottom: the state; mean increase $0.07$;
    the expected certificate increment, changing from $-0.021$ to
    $0.049$, and allowance $-0.001$; and logarithmic monitor wealth.
    Temporal prediction reduces the influence of observations collected
    before onset as new scores arrive.
    U, T-equal and S-equal are shown with both M and L stakes,
    as in Figure~\ref{fig:synthetic:ex-persistent}.
    The wealth threshold is $20$.
    Neither the history nor wealth is reset at onset.}
    \label{fig:synthetic:ex-delay}
\end{figure}

\begin{figure}[!htbp]
    \centering
    \includegraphics[width=\linewidth]{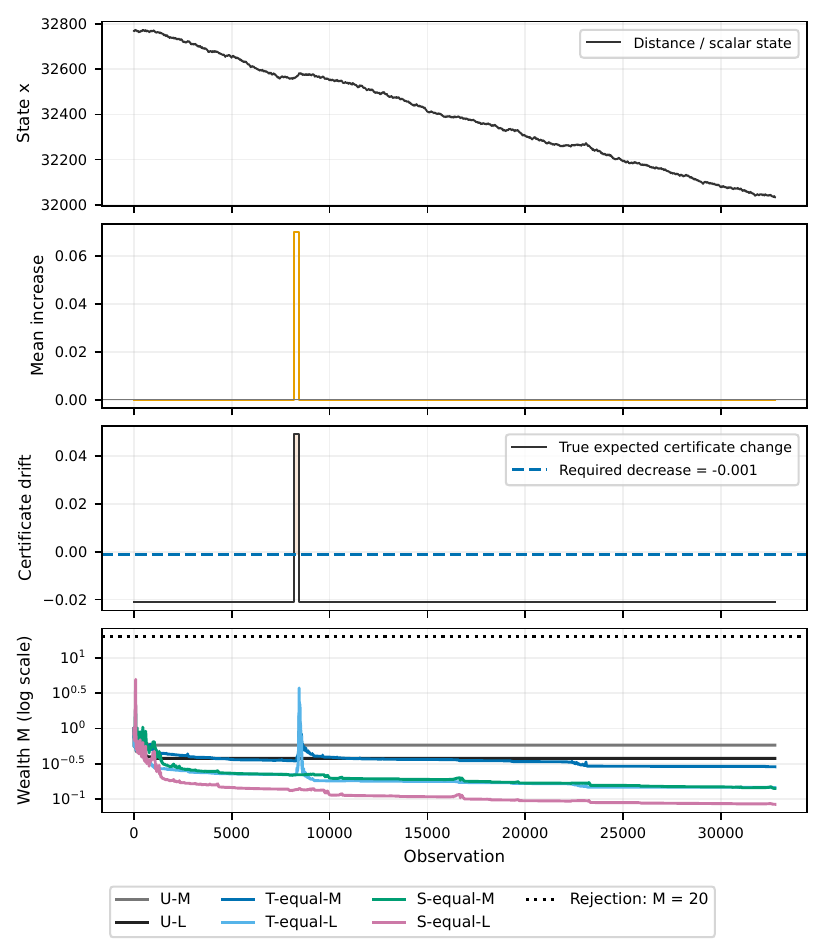}
    \caption{Temporary violation lasting $256$ observations from onset $8192$,
    with seed $970000003$.
    The conditional mean of the corrected increment is $0.05$ during
    the disturbance and $-0.02$ before and after it.
    From top to bottom: the state; temporary mean increase $0.07$;
    the expected certificate increment and allowance $-0.001$;
    and logarithmic wealth for U and the equal T/S mixtures with
    M and L stakes.
    None of the six displayed monitors reaches the dotted threshold $20$.
    The path is shown through observation $32768$, including recovery.
    The aggregate detection summaries also count eligible first alarms
    occurring after recovery.}
    \label{fig:synthetic:ex-temporary}
\end{figure}

\begin{figure}[!htbp]
    \centering
    \includegraphics[width=\linewidth]{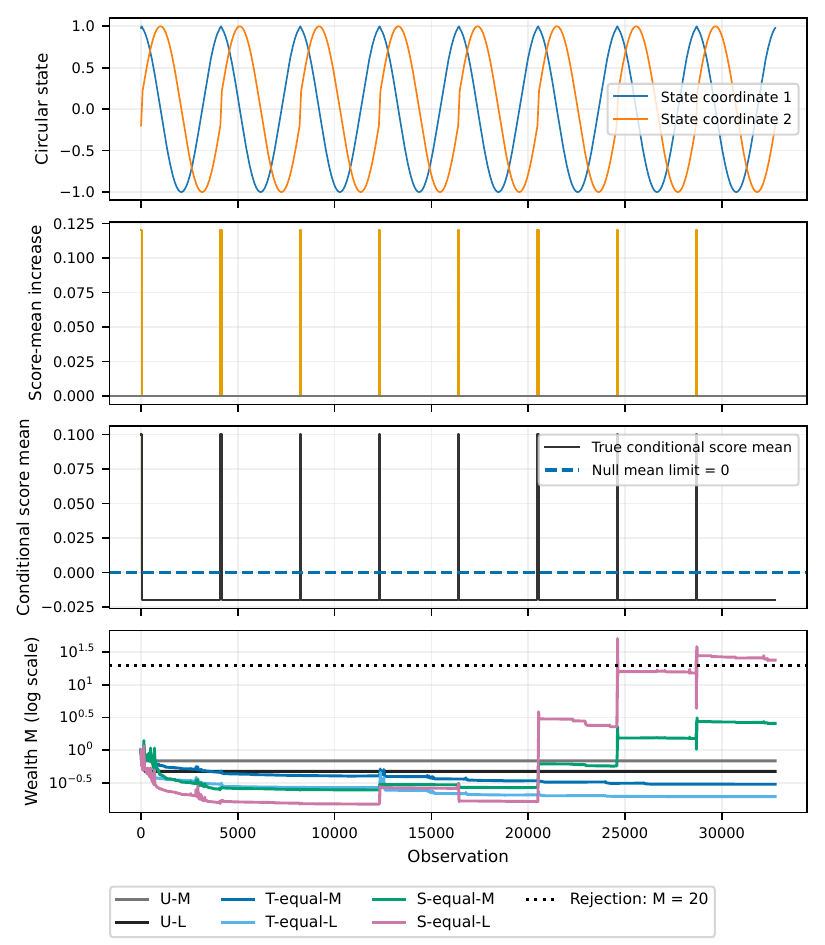}
    \caption{Direct recurrence with period $P=4096$ and seed $970000003$.
    Each of eight visits to a fixed circular sector contains $64$
    observations with score mean $0.10$.
    Between visits, the mean is $-0.02$.
    Independent symmetric score noise has amplitude $0.3$.
    From top to bottom: the two state coordinates; local mean increase
    $0.12$; the prescribed conditional score mean relative to zero;
    and logarithmic wealth for U and the equal T/S mixtures with
    M and L stakes.
    The $32768$ observations include exactly $512$ with positive
    conditional mean. Spatial prediction can reuse observations from
    earlier visits despite the long gaps between them.
    The dotted line marks wealth threshold $20$.}
    \label{fig:synthetic:ex-reuse}
\end{figure}

\begin{figure}[!htbp]
    \centering
    \includegraphics[width=\linewidth]{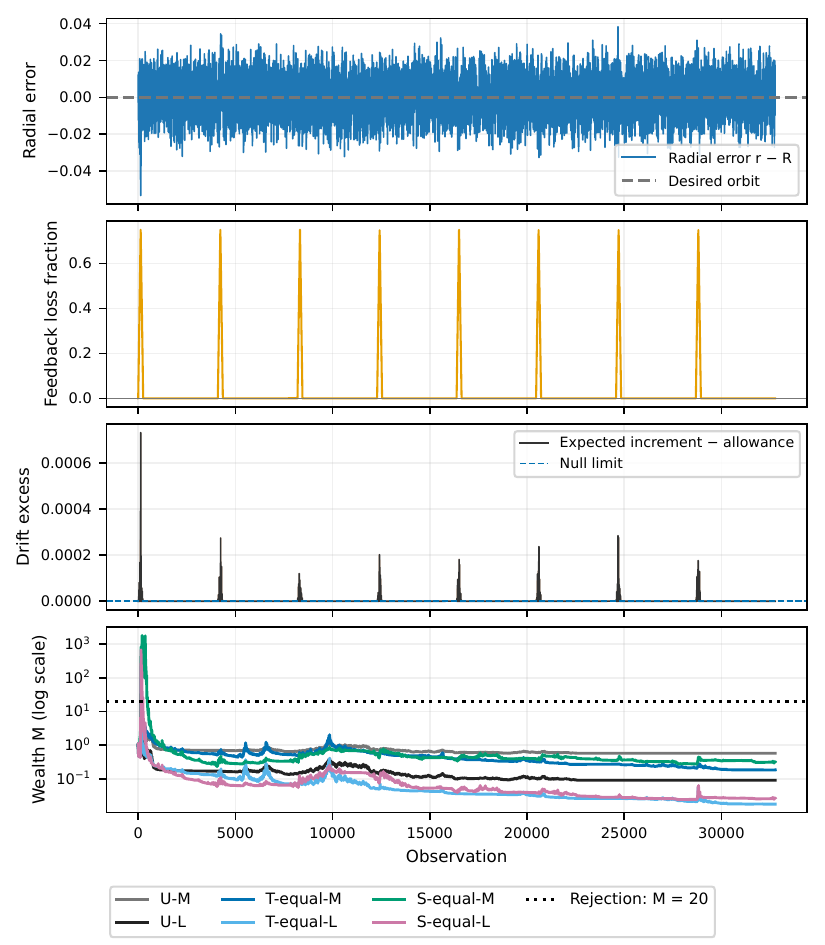}
    \caption{Stochastic ring with nominal period $P=4096$ and seed $970000003$.
    From top to bottom: radial error relative to the unit circle;
    the state-dependent reduction in restoring feedback, reaching
    $0.75$ in the affected sector; the exact expected certificate
    increment minus its nominal allowance; and logarithmic monitor wealth.
    The certificate is the squared radial error.
    Radial and angular noise make return times and violation exposure
    vary across runs.
    U and the equal T/S mixtures are shown with M and L stakes.
    The dotted line marks wealth threshold $20$.
    The expected-increment condition can be violated even while the
    process remains close to the circle.
    All observations are retained, including those after alarms.}
    \label{fig:synthetic:ex-ring}
\end{figure}

\begin{figure}[!htbp]
    \centering
    \includegraphics[width=\linewidth]{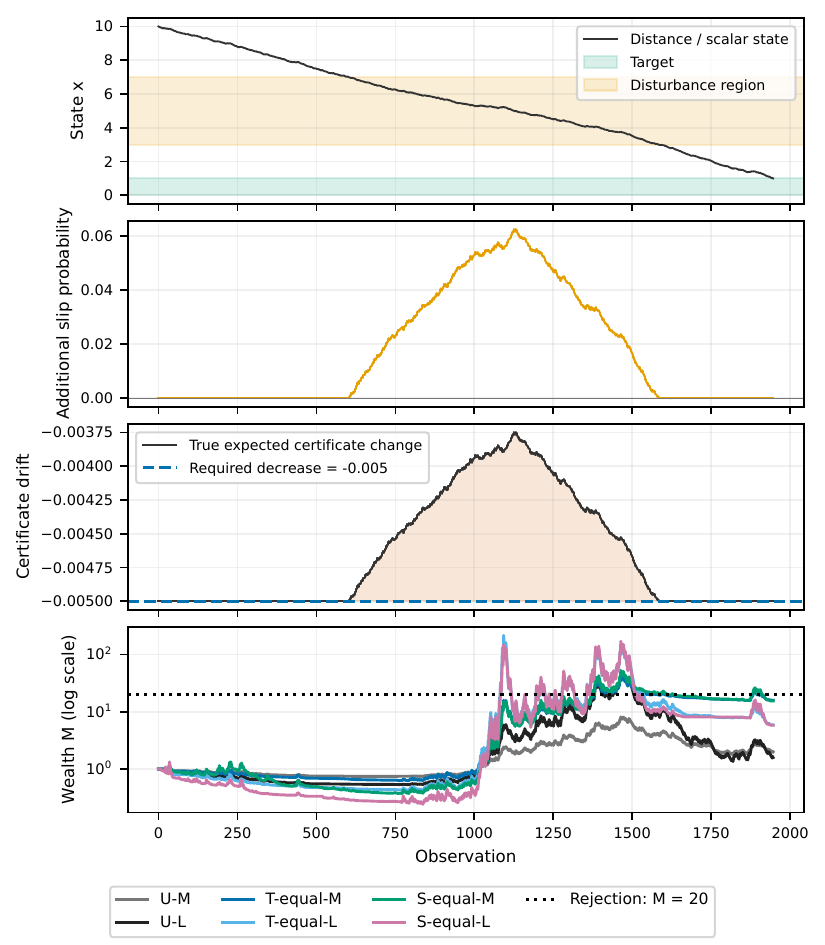}
    \caption{Hill disturbance with amplitude $a=0.0625$ and seed $970000003$.
    From top to bottom: distance to the target $[0,1]$, with the hill
    region $(3,7)$ shaded; additional slip probability
    $a(1-|x-5|/2)_+$; the exact expected certificate increment and
    required allowance $-0.005$; and logarithmic wealth for U and
    the equal T/S mixtures with M and L stakes.
    At $x=5$, the expected increment is $-0.00375$, a $25\%$
    reduction in progress.
    The robot therefore still approaches the target in expectation
    while violating the required rate of decrease.
    The dotted line marks wealth threshold $20$.
    The run ends when the robot reaches the target, not when a monitor
    raises an alarm.}
    \label{fig:synthetic:ex-hill}
\end{figure}

\begin{figure}[!htbp]
    \centering
    \includegraphics[width=\linewidth]{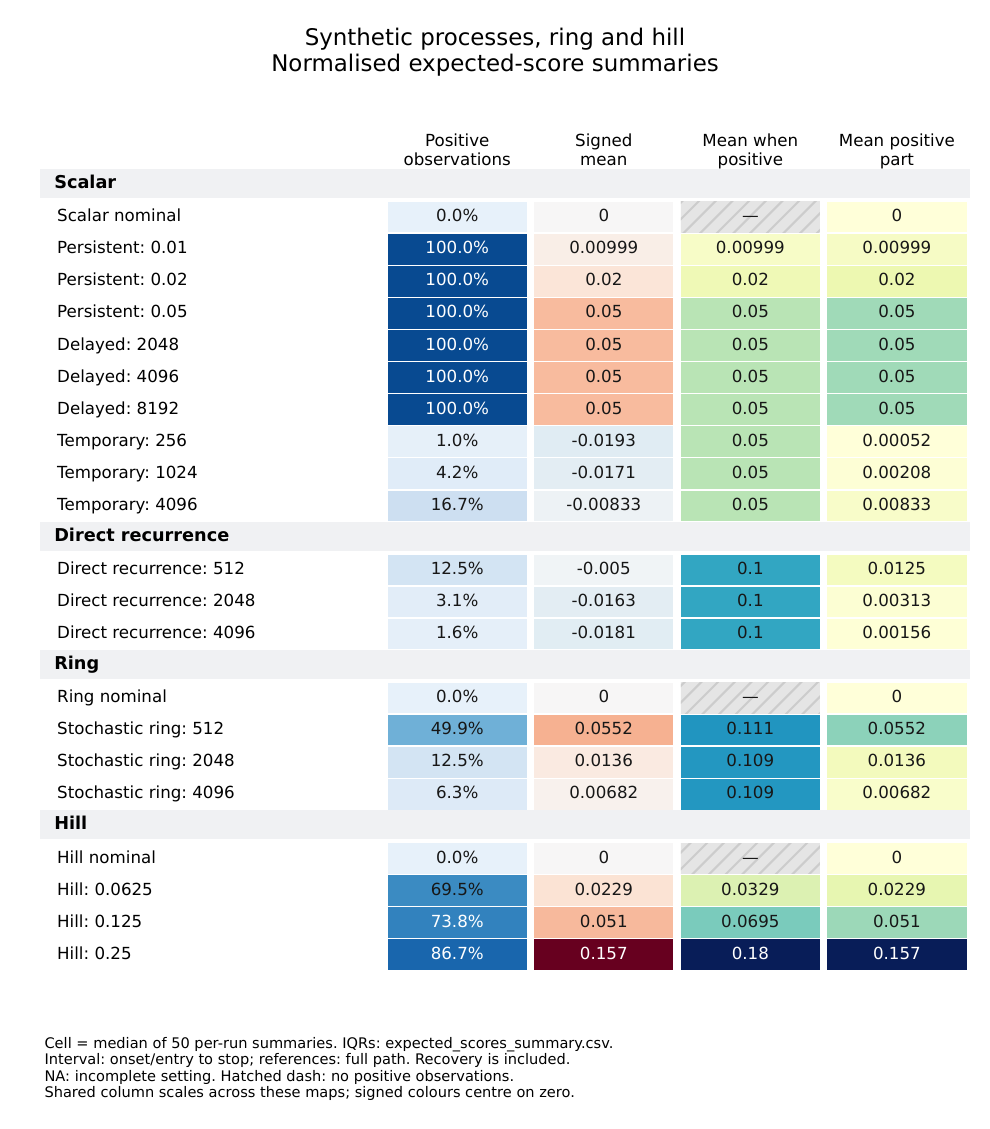}
    \caption{Normalised expected-score summaries for the synthetic benchmarks.
    Each cell is the median of $50$ quantities computed separately
    for each run: the fraction of observations with $\mu_n>0$,
    the signed mean of $\mu_n$, its mean over positive-mean observations,
    and the mean of $(\mu_n)_+$ over all observations.
    Summaries start at onset or first entry and include recovery and
    later excursions. References use the full path.
    Each run receives equal weight.
    A hatched dash marks an undefined mean over positive observations;
    NA denotes incomplete data.
    Column colour scales match the control heatmap, with the signed
    scale centred at zero.}
    \label{fig:synthetic:expected_score}
\end{figure}

\clearpage
\section{Control}
\label{sec:selected-control-figures}

\begin{figure}[!htbp]
    \centering
    \includegraphics[width=\linewidth]{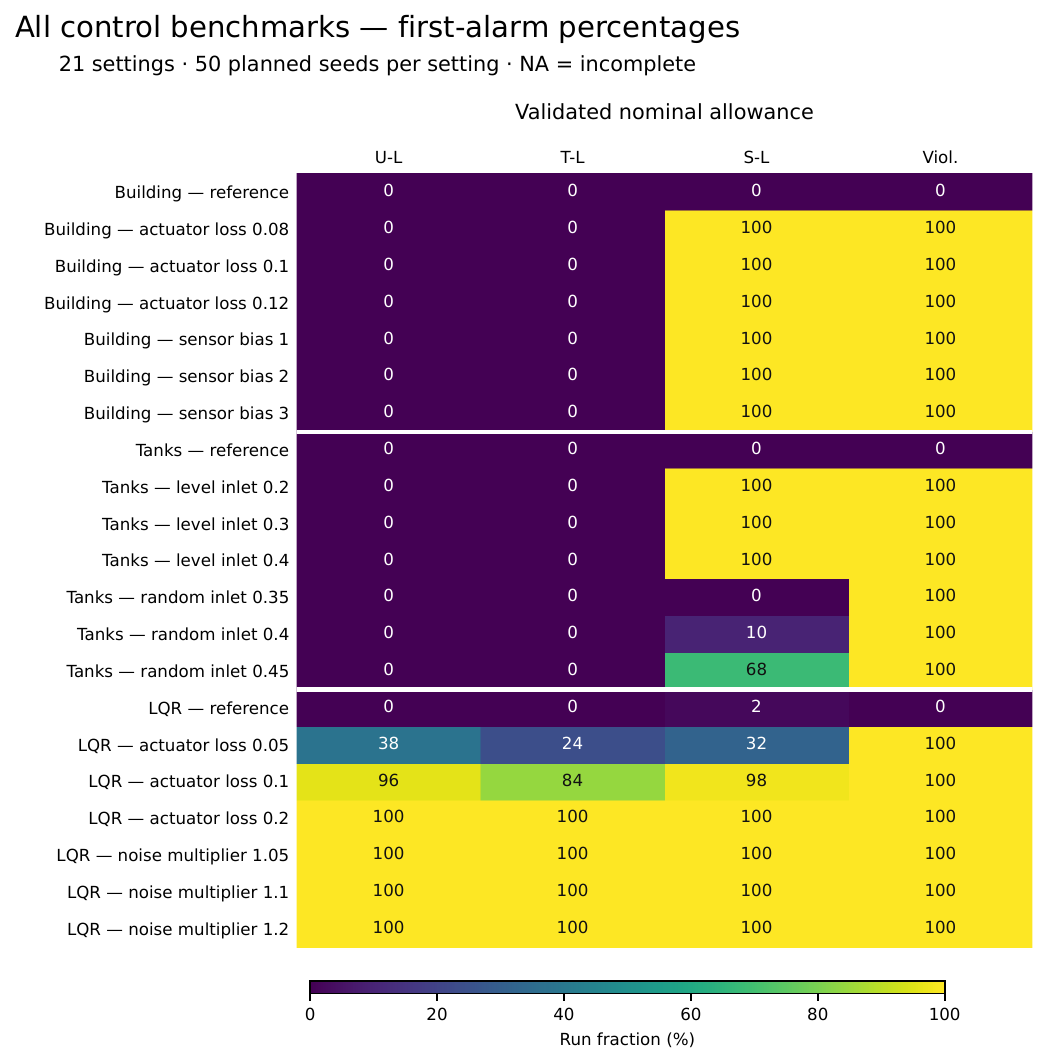}
    \caption{First-alarm percentages for Building, Two Tanks and LQR.
    Each system has a reference setting and two disturbances, each
    evaluated at three strengths.
    U-L, T-L and S-L use uniform, temporal and spatial prediction
    with log-return stakes and wealth threshold $20$.
    Every cell uses all $50$ paired runs.
    The \textnormal{Viol.} column gives the percentage of trajectories
    encountering at least one positive conditional excess, computed
    from exact model moments with tolerance $10^{-12}$ in raw units.
    This differs from both the alarm rate and the fraction of
    observations with positive conditional excess.
    Reference alarms measure false alarms under the validated
    nominal conditions.}
    \label{fig:control:overview-rate}
\end{figure}

\begin{figure}[!htbp]
    \centering
    \includegraphics[width=\linewidth]{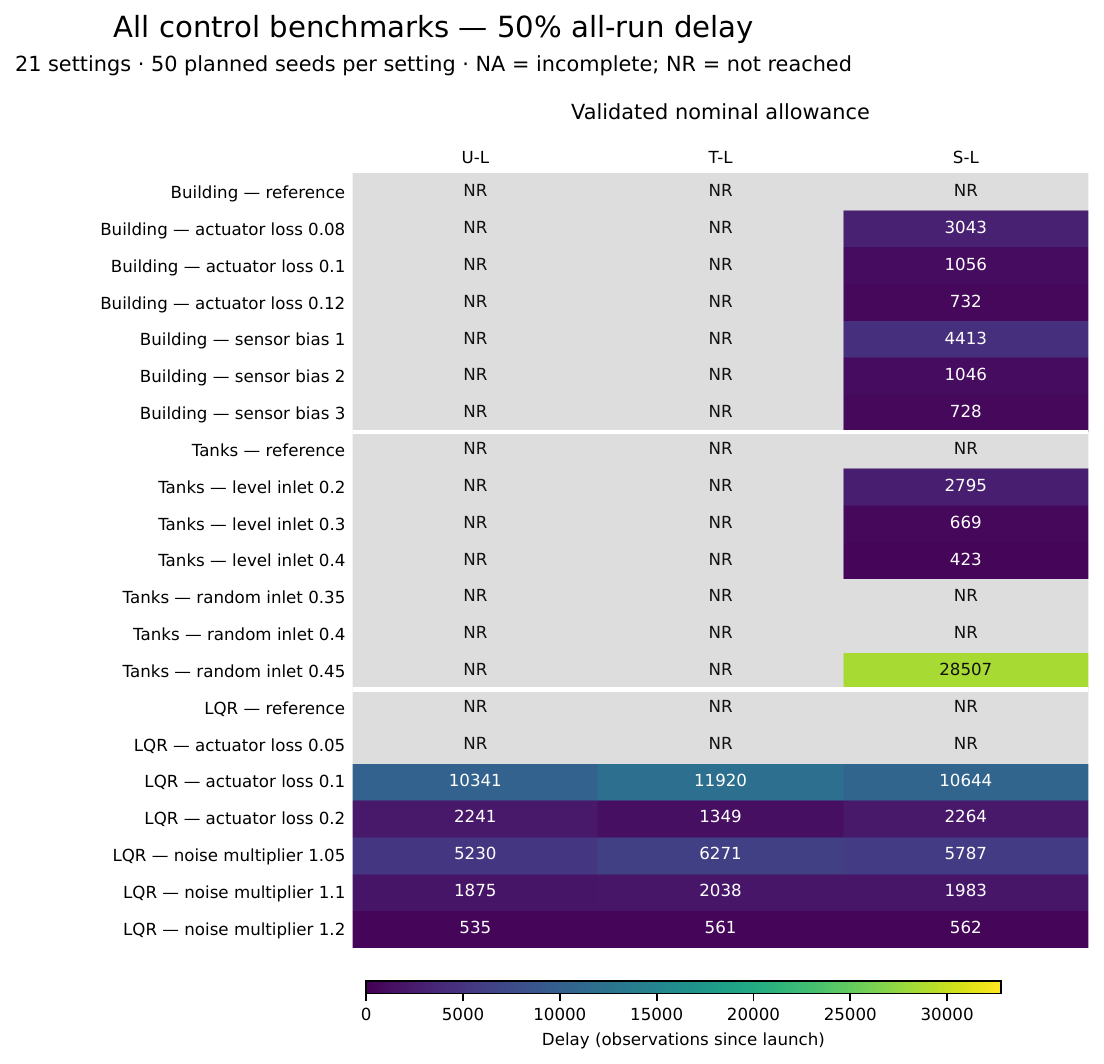}
    \caption{The $50\%$ detection delays for the control settings and log-return
    monitors in Figure~\ref{fig:control:overview-rate}.
    Disturbances act from the start, so each entry is the number of
    observations until $25$ of all $50$ runs have raised their first alarm.
    NR means this fraction was not reached, and NA denotes incomplete data.
    Nondetections remain in the denominator.
    Reference delays also start at the beginning of the run and concern
    false alarms.
    LQR observations are spaced by $0.1$ physical time units;
    all entries are reported in observations.}
    \label{fig:control:overview-delay}
\end{figure}

\begin{figure}[!htbp]
    \centering
    \includegraphics[width=\linewidth]{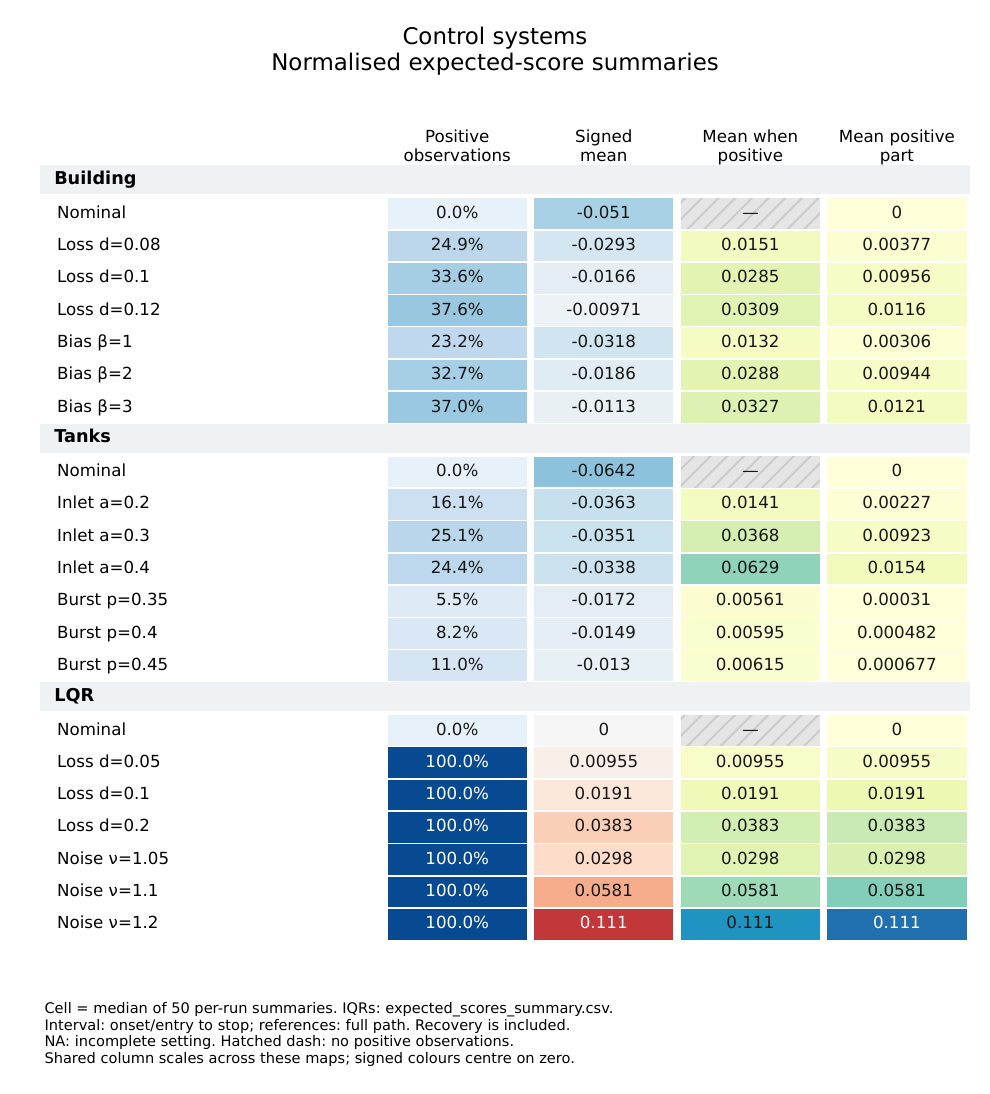}
    \caption{Normalised expected-score summaries for the control benchmarks,
    computed from exact conditional moments along each recorded path.
    Cells give medians of $50$ equally weighted per-run summaries:
    the fraction of positive-mean observations, the signed mean,
    the mean over positive-mean observations, and the mean positive part.
    The last quantity includes all observations in its denominator.
    A negative signed mean can therefore coexist with local violations.
    Positivity requires a raw conditional excess above $10^{-12}$.
    Reference and disturbed runs use their full paths, including
    observations after alarms.
    A hatched dash marks an undefined mean over positive observations;
    NA denotes incomplete data.
    Column colour scales match Figure~\ref{fig:synthetic:expected_score}.}
    \label{fig:control:expected_score}
\end{figure}

\clearpage
\section{Neural}
\label{sec:selected-neural-figures}

\begin{figure}[!htbp]
    \centering
    \includegraphics[width=\linewidth,height=0.75\textheight,keepaspectratio]{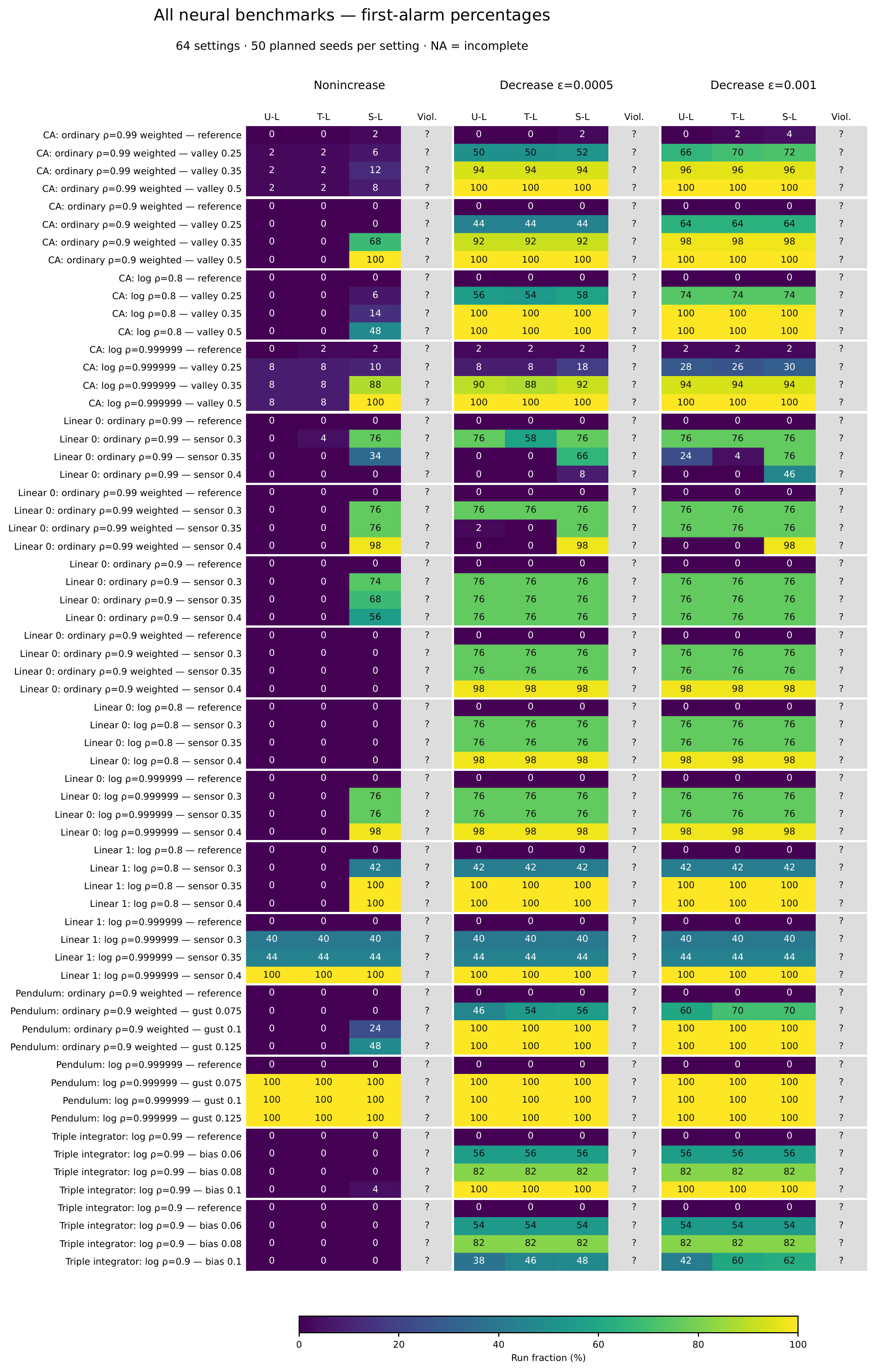}
    \caption{First-alarm percentages for the neural benchmarks.
    Each of the $16$ original controller--certificate pairs has a reference
    setting and three disturbance strengths.
    Each row uses $50$ trajectories shared by U-L, T-L and S-L and by
    the conditions $\epsilon=0,0.0005,0.001$.
    The allowance is $-\epsilon$ for ordinary certificates and
    $(e^{-\epsilon}-1)f(x)$ for exponentiated logarithmic certificates.
    All monitors use wealth threshold $20$.
    Question marks in \textnormal{Viol.} indicate that the conditional
    mean estimates at selected states do not determine the exact fraction
    of runs encountering a violation.
    All three conditions remain unverified under the adapted dynamics,
    including in the reference settings.
    Reference alarms therefore do not measure false-alarm rates under
    a verified null.}
    \label{fig:neural:overview-rate}
\end{figure}

\begin{figure}[!htbp]
    \centering
    \includegraphics[width=\linewidth,height=0.75\textheight,keepaspectratio]{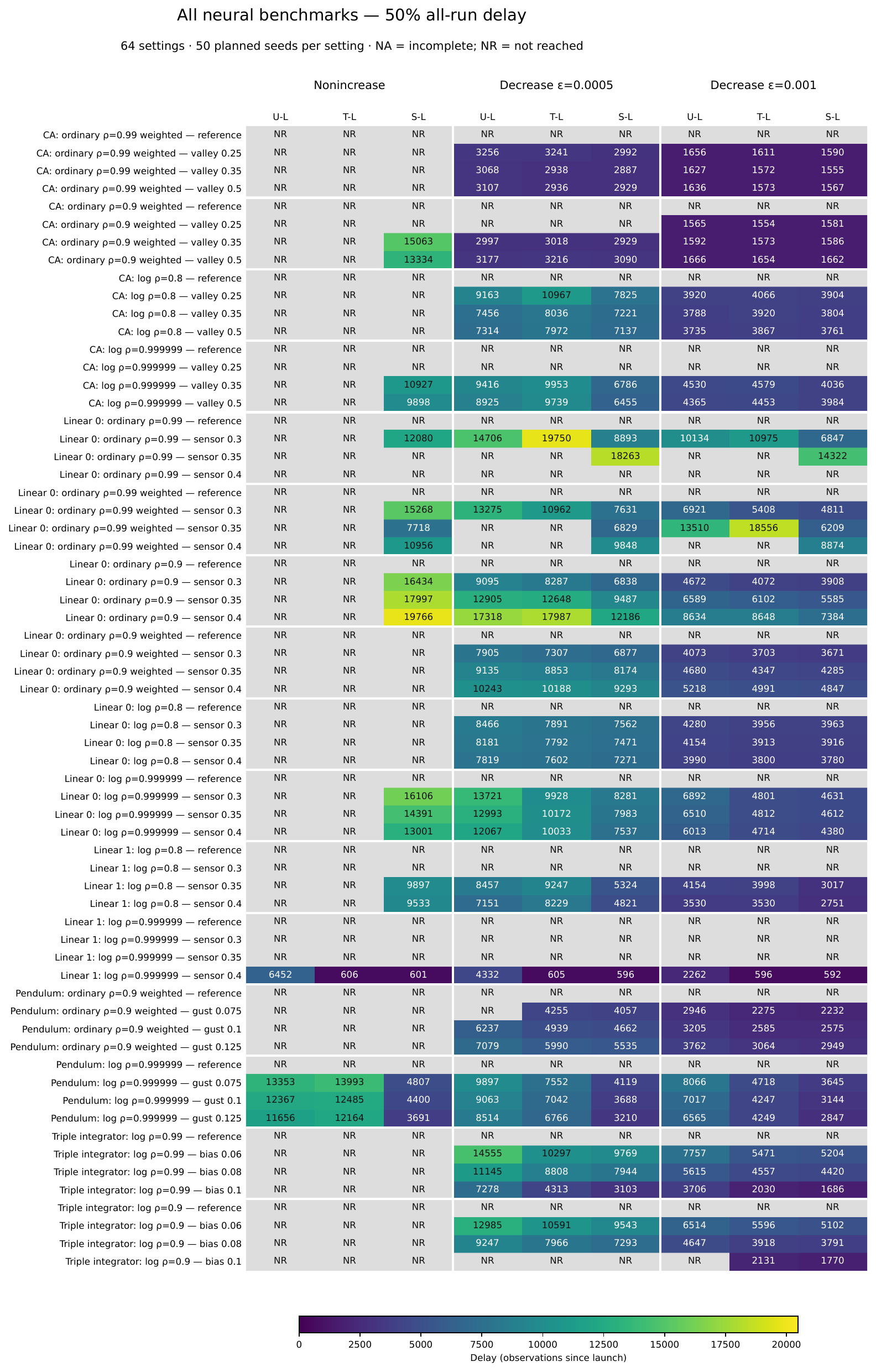}
    \caption{The $50\%$ detection delays for the neural pairs, disturbances and
    conditions in Figure~\ref{fig:neural:overview-rate}, measured in
    observations from the start of each run.
    Each finite entry requires $25$ first alarms among all $50$ paired
    paths. Short paths and nondetections remain in the denominator.
    NR means this fraction was not reached, and NA denotes incomplete data.
    Paths end at the target, an unsafe set, domain exit, certificate
    threshold or $20480$ observations.
    Detection therefore depends on both the evidence available and the
    length of the trajectory.
    All condition blocks share a colour scale.
    The different conditions change the allowance on the same
    physical trajectories.}
    \label{fig:neural:overview-delay}
\end{figure}

\clearpage

\end{document}